\documentclass[10pt]{article}
\usepackage[top=1in, bottom=1in, left=0.9in, right=0.9in]{geometry}

\usepackage{zhiyuz}

\title{\textbf{Optimal Comparator Adaptive Online Learning with Switching Cost}}
\author{
  Zhiyu Zhang \\
  Boston University\\
  \texttt{zhiyuz@bu.edu}\\
  \and
  Ashok Cutkosky \\
  Boston University\\
  \texttt{ashok@cutkosky.com}\\
  \and
  Ioannis Ch. Paschalidis\\
  Boston University\\
  \texttt{yannisp@bu.edu}\\
}
\date{\vspace{-5ex}}

\begin{document}
\maketitle

\begin{abstract}
Practical online learning tasks are often naturally defined on unconstrained domains, where optimal algorithms for general convex losses are characterized by the notion of \emph{comparator adaptivity.} In this paper, we design such algorithms in the presence of switching cost -- the latter penalizes the typical optimism in adaptive algorithms, leading to a delicate design trade-off. Based on a novel \emph{dual space scaling} strategy discovered by a continuous-time analysis, we propose a simple algorithm that improves the existing comparator adaptive regret bound \cite{zhang2022adversarial} to the optimal rate. The obtained benefits are further extended to the expert setting, and the practicality of the proposed algorithm is demonstrated through a sequential investment task.
\end{abstract}

\section{Introduction}\label{section:intro}

Online learning \cite{cesa2006prediction,hazan2016introduction,orabona2019modern} is a powerful framework for modeling sequential decision making tasks, such as neural network training, financial investment and robotic control. In each round, an agent picks a prediction $x_t$ in a convex domain $\calX$, receives a convex and Lipschitz loss function $l_t$ that depends on $x_1,\ldots,x_t$, and suffers the loss $l_t(x_t)$. The goal is to ensure that in any environment, the cumulative loss of the agent is never much worse than that of any fixed decision $u\in\calX$. That is, one aims to upper-bound the regret
\begin{equation*}
\sum_{t=1}^T\spar{l_t(x_t)-l_t(u)},
\end{equation*}
for all time horizon $T\in\N_+$, comparator $u\in\calX$ and loss sequence $l_1,\ldots,l_T$. 

If there exists a best fixed decision $u^*=\max_x\sum_{t=1}^Tl_t(x)$ in hindsight, then the regret with respect to any $u\in\calX$ is dominated by the regret with respect to $u^*$. In the context of training machine learning models, $u^*$ corresponds to the model parameter that minimizes the training error. Hence, intuitively, the regret bound characterizes how fast the algorithm finds $u^*$ through training. 

Most classical online learning algorithms are \emph{minimax} in nature, only tuned to optimize the worst-case regret. For example, if the domain $\calX$ is bounded, then the maximum distance between the best fixed decision $u^*$ and the initialization $x_1$ of the algorithm is the diameter $D$ of $\calX$. Only considering this worst case, it suffices to use \emph{Online Gradient Descent} (OGD) \cite{zinkevich2003online} with learning rate $\eta_t\propto D/\sqrt{t}$. The result is a $O(D\sqrt{T})$ regret bound that holds uniformly for all comparators $u\in\calX$. Such a minimax reasoning is prevalent, but limited in two substantial ways. 

\begin{enumerate}
\item It requires a bounded domain. Many practical problems are naturally unconstrained, making such arguments inapplicable. 
\item Practical applications are usually not the worst case, which means that the minimax bound $O(D\sqrt{T})$ is typically loose. Think about the situation where we have a prior guess of $u^*$, from either domain knowledge or pre-training. Using this prior as the initial prediction $x_1$, we should expect a provable gain over arbitrarily initializing the algorithm -- specifically, had we known the \emph{correct} distance $\norm{u^*-x_1}$, we could have picked $\eta_t\propto\norm{u^*-x_1}/\sqrt{t}$ in OGD, resulting in $O(\norm{u^*-x_1}\sqrt{T})$ regret. Achieving this goal \emph{without the prior knowledge of $\norm{u^*-x_1}$} is of both theoretical and practical importance. 
\end{enumerate}

Recent studies of \emph{comparator adaptive} online learning\footnote{Also called \emph{parameter-free} online learning due to historical reasons.} \cite{luo2015achieving,orabona2016coin,cutkosky2018black} aim to address these issues. The domain does not need to be bounded, and the regret bound is $\tilde O(d(u,x_1)\sqrt{T})$, where $d(\cdot,\cdot)$ is some suitable distance measure. Intuitively, these algorithms are both \emph{optimistic} and \emph{robust}: given prior information on $u^*$, we can pick $x_1$ such that $d(u^*,x_1)$, and consequently the regret bound, are both low. Meanwhile, even when our initialization $x_1$ is wrong (i.e., $d(u^*,x_1)$ is large), the regret bound is still \emph{almost as good} (up to logarithmic factors) as that of the minimax algorithm with the best tuning in hindsight. Such properties have shown benefits in diverse applications, e.g., \cite{orabona2017training,jun2019parameter,van2019user}. 

In this paper, we extend the design of these algorithms to a classical setting with switching costs. Here the agent is penalized not only by its loss, but also by how fast it changes its predictions. Practically, switching costs are useful whenever the smooth operation of a system is favored, such as in network routing, control of electrical grid, portfolio management with transaction costs, etc. Recently they also naturally show up in online decision problems with \emph{long term effects}, such as \emph{nonstochastic control} \cite{agarwal2019online}. With a given weight $\lambda\geq 0$ and a norm $\norms{\cdot}$, our goal is to guarantee a comparator adaptive bound for the \emph{augmented regret}
\begin{equation*}
\sum_{t=1}^T\spar{l_t(x_t)-l_t(u)}+\lambda\sum_{t=1}^{T-1}\norm{x_t-x_{t+1}}.
\end{equation*}

While gradient descent can incorporate switching costs by simply scaling its learning rate, extending comparator adaptive algorithms is a lot harder. Just like other adaptive algorithms \cite{duchi2011adaptive,daniely2015strongly}, the key idea of comparator adaptivity is to quickly respond to the incoming information and hedge aggressively. Switching costs, on the other hand, encourage the agent to stay still. Therefore, achieving our goal requires a delicate balance between these two opposite considerations. 

Similar trade-offs between adaptivity and switching costs have led to intriguing results in the past. For example, Gofer \cite{gofer2014higher} showed that the gradient variance adaptivity well-studied in the switching-free setting is impossible with normed switching costs, thus establishing a clear separation caused by the latter. Daniely and Mansour \cite{daniely2019competitive} showed that a common analytical technique for switching costs is incompatible to the so-called \emph{``strong adaptivity''} (i.e., a form of adaptivity w.r.t. nonstationary comparators). Regarding comparator adaptivity, our prior work \cite{zhang2022adversarial} proposed the first comparator adaptive algorithm with switching costs, but the obtained regret bound does not achieve the optimality criterion of the switching-free setting. The present paper closes this gap. 

\subsection{Contribution}

We develop comparator adaptive algorithms for two classical settings: ($i$) \emph{Online Linear Optimization} (OLO) with switching cost; ($ii$) \emph{Learning with Expert Advice} (LEA) with switching cost.

\begin{enumerate}
\item For one-dimensional unconstrained OLO with switching costs, assuming loss gradients $\abs{g_t}\leq 1$ and initial prediction\footnote{For general $x_1$, we can replace $\abs{u}$ in the regret bound by $\abs{u-x_1}$.} $x_1=0$, we propose an algorithm that guarantees
\begin{equation*}
\sum_{t=1}^Tg_t(x_t-u)+\lambda\sum_{t=1}^{T-1}\abs{x_t-x_{t+1}}\leq C\sqrt{\lambda T}+\abs{u}O\rpar{\sqrt{\lambda T\log(C^{-1}\abs{u})}},
\end{equation*}
where $C>0$ is any hyperparameter chosen by the user (Section~\ref{section:olo}). Our bound achieves several forms of optimality with respect to $\lambda$, $\abs{u}$ and $T$, improving the prior work \cite{zhang2022adversarial}. Extensions to bounded domains and general dimensional domains are presented, which include some new technical components (Appendix~\ref{section:extension}).

\item Converting the above result from OLO to LEA, we demonstrate how classical conversion techniques \cite{luo2015achieving,orabona2016coin} are \emph{designed} to have large switching costs, and then propose a fix with a clear geometric interpretation. This leads to the first comparator adaptive algorithm for LEA with switching costs (Section~\ref{section:lea}). 
\end{enumerate}

Technically, our improvement over \cite{zhang2022adversarial} relies on a novel \emph{dual space scaling} strategy. This is actually not guessed, but \emph{systematically discovered} by a continuous-time analysis (Section~\ref{subsection:continuous}), whose procedure follows another prior work of ours \cite{zhang2022pde}. In the continuous-time limit, it becomes evident what kinds of algorithmic structures from the switching-free setting are transferable to the setting with switching costs. Indeed, revealing generalizable knowledge is a key benefit of the continuous-time analysis, which was not demonstrated in \cite{zhang2022pde}. As an added bonus, both our OLO algorithm and its analysis are considerably simpler than those from \cite{zhang2022adversarial}. 

Concluding these theoretical results, our OLO algorithm is applied to a portfolio management task with transaction costs (Appendix~\ref{section:financial}). Numerical results support its superiority over the existing approach \cite{zhang2022adversarial}. 

\subsection{Related work}

\paragraph{Online learning basics}Throughout this paper we will only consider linear losses. The generality of our setting is preserved, since convex losses can be reduced to linear losses through the relation $\sum_{t=1}^T[l_t(x_t)-l_t(u)]\leq \sum_{t=1}^T\inner{\nabla l_t(x_t)}{x_t-u}$ \cite{hazan2016introduction,orabona2019modern}. Online learning with linear losses is called \emph{Online Linear Optimization} (OLO). As its important special case, \emph{Learning with Expert Advice} (LEA) considers OLO on a probability simplex, but aims at a different form of regret bound due to its different geometry.

Classical minimax approaches in online learning include \emph{Online Mirror Descent} (OMD) and \emph{Follow the Regularized Leader} (FTRL), with \emph{Online Gradient Descent} (OGD) being their most well-known special case. We write ``gradient descent'' as the minimax baseline for the ease of exposition. Moreover, both OMD and FTRL have elegant duality interpretations \cite[Section 6.4.1 and 7.3]{orabona2019modern}, involving simultaneous updates on the primal space (the domain $\calX$) and the dual space (the space of gradients). We will exploit this duality in our analysis.

\paragraph{Comparator adaptive online learning} Also known as \emph{parameter-freeness}, comparator adaptive online learning aims at matching the performance of the optimally-tuned gradient descent in hindsight, without knowing the correct tuning parameter. The associated regret bound can appear in different forms, depending on the specific learning setting. 
\begin{enumerate}
\item For LEA, a comparator adaptive bound has the form $O\rpar{\sqrt{T\cdot \kl(u||\pi)}}$, where $u$ and $\pi$ are distributions on the expert space representing the comparator and a user-chosen prior. Such an idea was initiated in \cite{chaudhuri2009parameter}, and the analysis was improved and extended by a series of works \cite{chernov2010prediction,luo2015achieving,koolen2015second,chen2021impossible,negrea2021minimax,portella2022continuous}. Notably, a comparator adaptive LEA algorithm naturally induces a bound on the \emph{$\eps$-quantile regret} -- the regret with respect to the $\eps$-quantile best expert. The latter is particularly meaningful when the number of experts is large. Lower bounds were considered in \cite{negrea2021minimax}.

We will present a nonasymptotic improvement of the $\sqrt{\kl}$ divergence in this paper. Frameworks that generalize root KL to \emph{$f$-divergences} have been studied in \cite{alquier2021non,negrea2021minimax}, but to our knowledge, no existing algorithm guarantees a better divergence term than root KL, even without switching costs. 

\item For OLO, typical comparator adaptive regret bounds are $C+\norm{u}O\rpar{\sqrt{T\log(C^{-1}\norm{u}_*T)}}$ or $C\sqrt{T}+\norm{u}O\rpar{\sqrt{T\log(C^{-1}\norm{u}_*)}}$, where a prior $x_1$ can be incorporated by letting $u\leftarrow u-x_1$. These two bounds are both \emph{Pareto-optimal} (see \cite{zhang2022pde} for a detailed explanation), as they represent different trade-offs on the \emph{loss} (the regret at $u=x_1$) and the \emph{asymptotic regret} (when $\norm{u-x_1}$ is large). Existing works \cite{mcmahan2014unconstrained,cutkosky2018black,foster2018online,mhammedi2020lipschitz,jacobsen2022parameter} were mostly independent of the LEA setting, but unified views were presented in \cite{foster2015adaptive,orabona2016coin}. Lower bounds were studied in \cite{streeter2012no,orabona2013dimension,zhang2022pde}. 
\end{enumerate}

\paragraph{Switching cost} Motivated by numerous applications, switching costs in online decision making have been studied from many different angles. For example, beside online learning, the online algorithm community has investigated settings like \emph{smoothed online optimization} \cite{chen2018smoothed,goel2019beyond,li2020online} and \emph{convex body chasing} \cite{bubeck2019competitively,sellke2020chasing}, where the loss function $l_t$ is observed \emph{before} the agent picks the prediction $x_t$. There, the switching cost is the key consideration that prevents the trivial strategy $x_t\in\argmin_x l_t(x)$. As for online learning, an additional complication is that $x_t$ (e.g., the investment portfolio) should be selected without knowing $l_t$ (e.g., tomorrow's stock price). 

Even within online learning, there are several ways to model switching costs. In cases like network routing, every switch means changing the packet route, which can be costly. Therefore, one needs a \emph{lazy} agent whose amount of switches (or its expectation) \cite{kalai2005efficient,geulen2010regret,altschuler2018online,chen2020minimax,sherman2021lazy} is as low as possible -- a good modeling candidate is $\bm{1}[x_t\neq x_{t+1}]$. Alternatively, one could take a \emph{smooth} view \cite{andrew2013tale,bhaskara2021power,wang2021online,zhang2021revisiting} where the agent can perform as many switches as it wishes, as long as the cumulative distance of switching is low -- in this view, switching cost can be a norm $\norm{x_t-x_{t+1}}$ or its smoothed variant $\norm{x_t-x_{t+1}}^2$. The present work primarily considers the $L_1$ norm switching cost motivated by the transaction cost in some financial applications. Notably, for LEA, the $L_1$ norm unifies the lazy view and the smooth view \cite[Section 5.2]{daniely2019competitive}. 

Although switching costs have been extensively studied, existing works on the combination of adaptivity and switching cost are quite sparse. As one should carefully trade off these two opposite requirements, there have been interesting impossibility results \cite{gofer2014higher,daniely2019competitive}, highlighted in our introduction. In this regard, one should not believe that every classical adaptivity can be naturally achieved with switching costs. The present paper shows that the optimal comparator adaptivity can indeed be achieved, thus improving the suboptimal result from \cite{zhang2022adversarial}.

\paragraph{Relation to downstream problems} More generally, incorporating switching costs amounts to considering a \emph{history-dependent} adversary: it can pick loss functions that depend not only on the instantaneous prediction $x_t$, but also on the previous prediction $x_{t-1}$. One could further generalize this setting to \emph{online learning with memory} \cite{cesa2013online,anava2015online}, where the loss depends on a fixed-length prediction history, and finally to \emph{dynamical systems} \cite{agarwal2019online,simchowitz2020improper,simchowitz2020making}, where the entire history matters. In fact, a common procedure in \emph{nonstochastic control} \cite{agarwal2019online} is to bound the risk in the future by a properly scaled switching cost. Achieving comparator adaptivity with switching costs may benefit these important downstream problems as well, by making algorithms \emph{easy to combine} \cite{cutkosky2019combining,cutkosky2020parameter,zhang2022adversarial}.

\paragraph{Continuous-time approach to online learning} Finally, on the technical side, our methodology builds on an emerging continuous-time perspective of online learning. From a theoretical angle, Kohn and Serfaty \cite{kohn2010deterministic} demonstrated a rigorous connection between \emph{Partial Differential Equations} (PDE) and discrete-time repeated games. More recently, such a connection has led to \emph{algorithmic benefits} in minimax LEA \cite{zhu2014two,rokhlin2017pde,drenska2020prediction,kobzar2020a_new,harvey2020optimal,bayraktar2020finite,bayraktar2020asymptotic,greenstreet2022efficient}, $\eps$-quantile bounds \cite{portella2022continuous} and comparator adaptive OLO \cite{zhang2022pde}. The key idea is that online learning algorithms in the continuous-time limit can be more easily parameterized and analyzed. In this paper we will show an additional bonus: generalizing algorithmic insights is also easier in the continuous-time limit. 

\subsection{Notation}

Let $f^*$ be the Fenchel conjugate of a function $f$. $\Delta(d)$ represents the $d$-dimensional probability simplex; $\kl$ and $\tv$ denote the KL divergence and the total variation distance, respectively. For two integers $a\leq b$, $[a:b]$ is the set of all integers $c$ such that $a\leq c\leq b$. $\log$ represents the natural logarithm when the base is omitted. Throughout this paper, ``increasing'' and ``positive'' are not strict (i.e., include equality as well). 

For a twice differentiable function $V(t,S)$ where $t$ represents time and $S$ represents a spatial variable, let $\nabla_tV$, $\nabla_{tt}V$, $\nabla_SV$ and $\nabla_{SS}V$ be the first and second order partial derivatives. In addition, we define discrete derivatives as
\begin{equation*}
\bar \nabla_t V(t,S)\defeq V(t,S)-V(t-1,S),
\end{equation*}
\begin{equation}\label{eq:def_deriv}
\bar \nabla_{S} V(t,S)\defeq \frac{1}{2}\spar{V(t,S+1)-V(t,S-1)},
\end{equation}
\begin{equation*}
\bar \nabla_{SS} V(t,S)\defeq V(t,S+1)+V(t,S-1)-2V(t,S).
\end{equation*}

\section{OLO with switching cost}\label{section:olo}

This section presents our main result, a comparator adaptive OLO algorithm with switching costs. We will focus on the 1D unconstrained setting. Extensions to general settings are deferred to Appendix~\ref{section:extension}.

\subsection{The 1D unconstrained setting}\label{subsection:1dsetting}

We consider the domain $\calX=\R$, a Lipschitz constant $G>0$ for the loss gradients, and a weight $\lambda\geq 0$ for switching costs. In the $t$-th round, the agent predicts $x_t\in\R$, receives a loss gradient $g_t\in[-G,G]$ that depends on past predictions $x_{1:t}$, and suffers an augmented loss $g_tx_t+\lambda\abs{x_t-x_{t-1}}$ (w.l.o.g., let $x_0=x_1=0$). The performance metric is the augmented regret for all $u\in\R$ and $T\in\N_+$:
\begin{equation}\label{eq:def}
\reg^\lambda_T(u)\defeq \sum_{t=1}^Tg_t(x_t-u)+\lambda\sum_{t=1}^{T-1}\abs{x_t-x_{t+1}}.
\end{equation}
Ignoring the dependence on $G$ for now, our goal is to show a comparator adaptive bound $\tilde O\rpar{\abs{u}\sqrt{\lambda T}}$, more specifically the optimal rates $C+\abs{u}O\rpar{\sqrt{\lambda T\log(C^{-1}\lambda\abs{u}T)}}$ or $C\sqrt{\lambda T}+\abs{u}O\rpar{\sqrt{\lambda T\log(C^{-1}\abs{u})}}$ for any hyperparameter $C>0$. These two cases are equivalent via the standard doubling trick \cite{shalev2011online}, as discussed in \cite{zhang2022pde}. 

For minimax algorithms like bounded domain gradient descent \cite{zinkevich2003online}, one can use scaled learning rates $\eta_t\propto 1/\sqrt{\lambda t}$ to ensure that both sums in (\ref{eq:def}) are $O\rpar{\sqrt{\lambda T}}$, thus obtaining a combined $O\rpar{\sqrt{\lambda T}}$ regret bound. However, such a divide-and-conquer approach does not apply to comparator adaptive algorithms, as one cannot \emph{separately} show the desirable bound on the two sums in (\ref{eq:def}). To see this, suppose one could guarantee the second sum alone is at most $1+\abs{u}O\rpar{\sqrt{T\log(\abs{u}T)}}$; here we only focus on the dependence on $\abs{u}$ and $T$. Since this cumulative switching cost is an algorithmic quantity \emph{independent of the comparator}, we can take infimum with respect to $u$ and obtain a ``budget'' of 1 for this sum. Following this argument, $\abs{x_T}\leq \abs{x_1}+\sum_{t=1}^{T-1}\abs{x_t-x_{t+1}}=O(1)$. That is, the algorithm should only predict around the origin, which clearly leads to large regret with respect to far-away comparators, under certain loss sequences. 

The challenge can be motivated in another way. As shown in \cite[Figure 9.1]{orabona2019modern}, the one-step switching cost $\abs{x_t-x_{t+1}}$ of comparator adaptive algorithms can grow exponentially with respect to $t$, whereas such a quantity is uniformly bounded in gradient descent. In fact, the exponential growth is the key mechanism for comparator adaptive algorithms to cover an unconstrained domain fast enough (thus improving minimax algorithms). This is however problematic when switching is also penalized, as one can no longer control the switching cost by uniformly scaling $\abs{x_t-x_{t+1}}$.

\subsection{Switching-adjusted potential}\label{subsection:dual_scale}

To address these issues, one should bound the switching cost and the standard OLO regret in a \emph{unified} framework, instead of treating them separately. Our prior work \cite{zhang2022adversarial} used the classical coin-betting approach from \cite{orabona2016coin,cutkosky2018black}, which is included in Appendix~\ref{subsection:suboptimal} for completeness. In the $t$-th round, the algorithm maintains a quantity $\wel_{t-1}$; by picking a \emph{betting fraction} $\beta_{t}\in[0,1]$, the prediction is set to $x_t=\beta_t\wel_{t-1}$. To further ensure low switching costs, the betting fraction $\beta_t$ is capped by a decreasing upper bound $O(1/\sqrt{t})$. Although it is analytically nontrivial, such a hard threshold is conservative, which could be the reason of our previous suboptimal result. 

In contrast, the present paper follows the \emph{potential framework} explored by a parallel line of works \cite{mcmahan2014unconstrained,foster2018online,mhammedi2020lipschitz,zhang2022pde}. Generally, these algorithms are defined by a potential function $V(t,S)$, where $t$ represents the time index, and $S$ represents a ``sufficient statistic'' that summarizes the history. In each round, the algorithm computes $S_{t-1}=-\sum_{i=1}^{t-1}g_i/G$, and the prediction $x_t$ is the derivative $\nabla_SV$ evaluated at $(t,S_{t-1})$. We will specifically consider Algorithm~\ref{alg:1d}, which is a variant based on the discrete derivative $\bar\nabla_S V$, cf. (\ref{eq:def_deriv}). 

\begin{algorithm*}[ht]
\caption{One-dimensional unconstrained OLO with switching costs.\label{alg:1d}}
\begin{algorithmic}[1]
\REQUIRE A hyperparameter $C>0$, the Lipschitz constant $G$, and a potential function $V(t,S)$ that implicitly depends on $\lambda$ and $G$. Initialize $S_0=0$.
\FOR{$t=1,2,\ldots$}
\STATE Predict $x_t=\bar\nabla_SV(t,S_{t-1})$, and receive the loss gradient $g_t$. Let $S_t=S_{t-1}-g_t/G$.
\ENDFOR
\end{algorithmic}
\end{algorithm*}

One could think of the potential framework as the dual approach of FTRL -- the potential function and the regularizer are naturally Fenchel conjugates. While the FTRL analysis relies on a one-step regret bound on the \emph{primal space} (the domain $\calX$, cf. \cite[Lemma 7.1]{orabona2019modern}), the potential framework constructs a similar one-step relation on the \emph{dual space} (the space of $S_t$, cf. \cite[Lemma 3.1]{zhang2022pde}). Along this interpretation, \textbf{our key idea is to incorporate switching costs by scaling on the dual space, rather than only on the primal space.} That is, given a potential function that works without switching costs, we scale the sufficient statistic sent to its second argument by a function of $\lambda$.

To better demonstrate this idea, let us first consider a quadratic potential $V(t,S)=(1/2)\cdot CGS^2$. The potential method suggests the prediction $x_t=\nabla_SV(t,S_{t-1})=C\sum_{i=1}^{t-1}g_i=x_{t-1}-Cg_{t-1}$, which is simply gradient descent with learning rate $C$. Scaling on the primal space means scaling $V$ directly, while scaling on the dual space means scaling the sufficient statistic $S$. It is clear that both cases are \emph{equivalent} to scaling the effective learning rate, which is the standard way to incorporate switching costs in bounded domain gradient descent. In other words, for this gradient descent potential, the two types of scaling are essentially the same. 

Now, to achieve optimal comparator adaptivity, we need a better potential where scaling on the dual space actually makes a difference. With a parameter $\alpha$ that will eventually depend on $\lambda$, we consider Algorithm~\ref{alg:1d} induced by the potential
\begin{equation}\label{eq:alpha}
V_\alpha(t,S)=C\sqrt{\alpha t}\spar{2\int_0^{S/\sqrt{4\alpha t}}\rpar{\int_0^u\exp(x^2)dx}du-1}. 
\end{equation}

When the Lipschitz constant $G=1$, it has been shown \cite{zhang2022pde} that $\alpha=1/2$ leads to comparator adaptivity without switching costs. Here we use $\alpha=4\lambda G^{-1}+2$, which amounts to scaling \emph{both} the primal space and the dual space: on the primal space, we scale up the overall prediction by $\Theta(\sqrt{\lambda G^{-1}+1})$, and on the dual space we scale down the sufficient statistic $S$ by $\Theta(1/\sqrt{\lambda G^{-1}+1})$. The latter gives us the optimal comparator adaptive bound (i.e., Pareto-optimal rate in $\abs{u}$ and $T$), while the former helps us obtain the optimal rate in $\lambda$. Due to incorporating $\lambda$ into the potential function $V_\alpha$, we call our approach the \emph{switching-adjusted potential method}. 

Although the dual space scaling strategy and the particular structure of $V_\alpha$ may seem mysterious at first glance, they are actually \emph{derived} from a continuous-time analysis. To proceed, we will first present the performance guarantee in the next subsection, and then revisit the derivation of this strategy in Section~\ref{subsection:continuous}.

\subsection{Optimal comparator adaptive bound}\label{subsection:sketch}

Despite its simplicity, our approach improves the result from our prior work \cite{zhang2022adversarial} by a considerable margin. 

\begin{restatable}{theorem}{main}\label{thm:main}
If $\alpha=4\lambda G^{-1}+2$, then Algorithm~\ref{alg:1d} induced by the potential $V_{\alpha}$ guarantees
\begin{equation*}
\mathrm{Regret}^\lambda_T(u)\leq \sqrt{(4\lambda G+2G^2) T}\spar{C+\abs{u}\rpar{\sqrt{4\log\rpar{1+\frac{\abs{u}}{C}}}+2}},
\end{equation*}
for all $u\in\R$ and $T\in\N_+$. 
\end{restatable}

Theorem~\ref{thm:main} simultaneously achieves several forms of optimality. 
\begin{enumerate}
\item Pareto-optimal loss-regret trade-off: considering the dependence on $u$ and $T$, $\reg^\lambda_T(u)=O\rpar{\abs{u}\sqrt{T\log\abs{u}}}$, while the \emph{cumulative loss} satisfies $\reg^\lambda_T(0)=O(\sqrt{T})$. An existing lower bound \cite[Theorem 10]{zhang2022pde} shows that even without switching costs, all algorithms satisfying a $O(\sqrt{T})$ loss bound must suffer a $\Omega\rpar{\abs{u}\sqrt{T\log\abs{u}}}$ regret bound. In this sense, our algorithm attains a \emph{Pareto-optimal} loss-regret trade-off, in a strictly generalized setting with switching costs. 
\item On $T$ alone: $\reg^\lambda_T(u)=O(\sqrt{T})$. Despite achieving comparator adaptivity, the asymptotic rate on $T$ is still the optimal one, matching the well-known minimax lower bound. 
\item On $\lambda$ alone: $\reg^\lambda_T(u)=O(\sqrt{\lambda})$. Our bound has the optimal dependence on the switching cost weight \cite[Theorem 5]{geulen2010regret}.
\end{enumerate}

To compare Theorem~\ref{thm:main} to \cite{zhang2022adversarial}, we have to convert them to the same loss-regret trade-off, i.e., both guaranteeing $\reg^\lambda_T(0)=O(1)$ or $\reg^\lambda_T(0)=O(\sqrt{T})$. Here we take the first approach -- details are presented in Appendix~\ref{subsection:conversion}. Let us only consider the dependence on $u$ and $T$.\footnote{Comparing the dependence on $\lambda$ is more subtle, as discussed in Appendix~\ref{subsection:suboptimal}.} By a doubling trick, our bound can be converted to $C+\abs{u}O\rpar{\sqrt{T\log(C^{-1}\abs{u}T)}}$, which improves the rate $C+\abs{u}O\rpar{\sqrt{T}\log(C^{-1}\abs{u}T)}$ from \cite[Theorem 1]{zhang2022adversarial}. Specifically, our converted upper bound also attains Pareto-optimality in this regime (i.e., matching the lower bound $\Omega\rpar{\abs{u}\sqrt{T\log(\abs{u}T)}}$ in \cite{orabona2013dimension}), whereas the existing approach does not. 

The proof of Theorem~\ref{thm:main} is sketched below, with the formal analysis deferred to Appendix~\ref{subsection:proof_main}. It mostly follows a standard potential argument, which is another benefit over the existing approach -- the idea of this proof is easier to interpret and generalize. 

\paragraph{Proof sketch of Theorem~\ref{thm:main}} To begin with, the first step is to show a one-step bound on the growth rate of the potential. If there is no switching cost, then the \emph{Discrete It\^{o} formula} can serve this purpose, which applies to any convex potential $V$. It is an established result in the probability literature \cite{kudvzma1982ito,fujita2008random,klenke2013probability}, and Harvey et al. \cite{harvey2020optimal} first applied it to minimax LEA. The version below is from our prior work, which is a small variant that removes the LEA context.

\begin{lemma}[Lemma 3.1 of \cite{zhang2022pde}]\label{lemma:ito}
If the potential function $V(t,S)$ is convex in $S$, then against any adversary, Algorithm~\ref{alg:1d} guarantees for all $t\in\N_+$,
\begin{equation*}
V(t,S_{t})-V(t-1,S_{t-1})\leq -G^{-1}g_{t}x_{t}+\bar\nabla_tV(t,S_{t-1})+(1/2)\cdot\bar\nabla_{SS}V(t,S_{t-1}).
\end{equation*}
\end{lemma}

Our key observation is the following lemma, which incorporates switching costs into $V_\alpha$. Note that the structure of $V_\alpha$ is important here. 

\begin{restatable}{lemma}{switch}\label{lemma:switch}
For all $\alpha>0$, consider Algorithm~\ref{alg:1d} induced by the potential $V_\alpha$. For all $t\in\N_+$, 
\begin{equation*}
\abs{x_t-x_{t+1}}\leq \bar\nabla_S V_{\alpha}(t,S_{t-1}+1)-\bar\nabla_S V_{\alpha}(t,S_{t-1}-1).
\end{equation*}
\end{restatable}

Combining the above, if we define
\begin{equation}\label{eq:residual}
\Delta_t\defeq \bar\nabla_tV_\alpha(t,S_{t-1})+\frac{1}{2}\bar\nabla_{SS}V_\alpha(t,S_{t-1})+G^{-1}\lambda\spar{\bar\nabla_S V_{\alpha}(t,S_{t-1}+1)-\bar\nabla_S V_{\alpha}(t,S_{t-1}-1)},
\end{equation}
then a telescopic sum yields a \emph{cumulative loss bound}
\begin{equation*}
\reg^\lambda_T(0)\leq \sum_{t=1}^T\rpar{g_tx_t+\lambda\abs{x_t-x_{t+1}}}\leq -G \cdot V_\alpha(T,S_T)+G\sum_{t=1}^T\Delta_t.
\end{equation*}

To proceed, we need to control the residual term $\Delta_t$, which may seem problematic due to its complicated form. Fortunately, a careful analysis shows that $\Delta_t$ \emph{vanishes} with a proper choice of $\alpha$.

\begin{restatable}{lemma}{vanish}\label{lemma:vanish}
If $\alpha\geq 4\lambda G^{-1}+2$, then for all $t$ and against any adversary, $\Delta_t\leq 0$. 
\end{restatable}

Finally, with the updated loss bound $\reg^\lambda_T(0)\leq -G\cdot V_\alpha(T,S_T)$, our regret bound follows from the classical loss-regret duality \cite{mcmahan2014unconstrained,orabona2019modern}.

\subsection{Continuous-time derivation}\label{subsection:continuous}

Now let us derive our dual space scaling strategy from a continuous-time perspective. Technically, the procedure is analogous to another prior work of ours \cite{zhang2022pde}, which studies optimal potential functions for the standard OLO setting without switching costs. Before starting, we need a generalized definition of the discrete derivative, with a tunable gap increment $\delta$. 
\begin{equation*}
\bar \nabla^\delta_{S} V(t,S)=\frac{1}{2\delta}\spar{V(t,S+\delta)-V(t,S-\delta)}.
\end{equation*}
Note that the choice of $\delta=1$ recovers $\bar \nabla_SV(t,S)$ in Algorithm~\ref{alg:1d}. The Lipschitz constant $G$ will be set to 1 for the ease of exposition. 

\paragraph{Step 1: discrete-time recursive inequality}

First, let us consider the following inequality that characterizes ``admissible'' potentials for Algorithm~\ref{alg:1d}. For all $t$ and $S$, 
\begin{equation}\label{eq:admissible}
V(t-1,S)\geq \max_{g\in[-1,1]}\left\{V(t,S-g)+g\bar \nabla^1_S V(t,S)+\lambda\abs{\bar \nabla^1_S V(t,S)-\bar \nabla^1_S V(t+1,S-g)}\right\}.
\end{equation}

Finding solutions of this inequality is sufficient for constructing regret bounds. To see this, suppose the above holds for some $V$. We can then plug in $S=S_{t-1}$ and guarantee that for all $g_t\in[-1,1]$,
\begin{equation*}
g_tx_t+\lambda\abs{x_t-x_{t+1}}\leq V(t-1,S_{t-1})-V(t,S_{t}).
\end{equation*}
A telescopic sum further leads to a cumulative loss bound $\reg^\lambda_T(0)\leq V(0,0)-V(T,S_T)$, and a regret bound on $\reg^\lambda_T(u)$ then follows from the standard loss-regret duality \cite{mcmahan2014unconstrained}. 

\paragraph{Step 2: $\eps$-scaled recursion} Since we ideally need \emph{optimal potential functions} that satisfy the inequality (\ref{eq:admissible}) without any slack, let us turn (\ref{eq:admissible}) into an \emph{equality} and try to approximately solve it. Intuitively this is a challenging task, as there is no natural way to parameterize the dependence of $V$ on the discrete time $t$. However, if we decrease the discrete time interval, solutions $V$ will be ``smoother'' and easier to describe. Concretely, let $\eps>0$ be a parameter that will later approach 0. On (\ref{eq:admissible}), we scale
\begin{enumerate}
\item the unit time by $\eps^2$;
\item the loss gradient $g$, the switching cost weight $\lambda$ and the gap increment by $\eps$.
\end{enumerate}

Both scaling factors are justified in the switching-free setting \cite[Appendix A.2]{zhang2022pde}. Notably, since $g$ and $\lambda$ have the same ``unit'', it is natural that they are scaled by the same rate. With that, we obtain a \emph{scaled recursion}
\begin{equation}\label{eq:scaled}
V(t-\eps^2,S)=\max_{g\in[-1,1]}\left\{V(t,S-\eps g)+\eps g\bar \nabla^\eps_S V(t,S)+\eps\lambda\abs{\bar \nabla^\eps_S V(t,S)-\bar \nabla^\eps_S V(t+\eps^2,S-\eps g)}\right\}.
\end{equation}

\paragraph{Step 3: continuous-time PDE} To proceed, we take the second-order Taylor approximation on all components of (\ref{eq:scaled}). Calculations are simple, and we defer the details to Appendix~\ref{subsection:detail_continuous}. Both the zeroth and the first order terms of $\eps$ naturally vanish. Only keeping the second order terms, we have
\begin{equation*}
\nabla_tV(t,S)+\max_{g\in[-1,1]}\rpar{\frac{1}{2}g^2\nabla_{SS}V(t,S)+\lambda\abs{g\nabla_{SS}V(t,S)}}=0.
\end{equation*}

As typical potential functions are convex in the sufficient statistic $S$, it is reasonable to impose an additional condition $\nabla_{SS}V(t,S)\geq 0$. Then, the above becomes the 1D \emph{backward heat equation}
\begin{equation*}
\nabla_tV+\alpha\nabla_{SS}V=0,
\end{equation*}
where $\alpha=\lambda+1/2$. Compared to the switching-free setting \cite[Eq.~5]{zhang2022pde}, we obtain the same PDE, but change the \emph{negative thermal diffusivity} $\alpha$ from $1/2$ to $1/2+\lambda$. This concisely characterizes the effect of switching costs on the structure of the online learning problem. 

\paragraph{Step 4: solving the PDE} The final step is to solve the backward heat equation. With a hyperparameter $c$, consider solutions of the form
\begin{equation*}
V_\alpha(t,S)=t^cg\rpar{\frac{S}{\sqrt{4\alpha t}}}.
\end{equation*}
Plug it in, the backward heat equation reduces to the Hermite \emph{Ordinary Differential Equation} (ODE)
\begin{equation*}
g''(z)-2zg'(z)+4cg(z)=0,
\end{equation*}
which is \emph{independent of $\alpha$}. This is a crucial observation, as it reveals the correct way to generalize the knowledge from the switching-free setting to the setting with switching costs. More specifically, 
\begin{itemize}
\item In the switching-free setting, we can take a solution $g(z)$ of the Hermite ODE, plug in the argument $z=S/\sqrt{2t}$ and obtain a potential function $V_\alpha$.
\item When switching costs are considered, the above derivation suggests us to take the \emph{same} function $g(z)$ as before, and plug in a scaled argument $z=S/\sqrt{4\alpha t}$. \textbf{This is precisely dual space scaling.}
\end{itemize}

Finally, as shown in \cite{zhang2022pde}, a particularly good choice of $c$ is $1/2$. Using this choice yields the switching-adjusted potential (\ref{eq:alpha}). 

\paragraph{Remark} To summarize, through this derivation we aim to demonstrate a key benefit of the continuous-time analysis: it makes the generalization of algorithmic structures easier. This was not presented in our prior work \cite{zhang2022pde}, but could be useful in the broader online learning context. 

Meanwhile, we do not intend to overclaim its strength -- although the continuous-time analysis provides useful intuition, we ultimately care about discrete-time regret bounds. Discretizing such arguments relies on an obscure argument that has not been made concrete yet: ``$V_\alpha$ derived in the continuous time also serves as a good potential in the discrete time.'' Indeed, verifying this property is technically nontrivial (Section~\ref{subsection:sketch}), and doing so requires a slightly more conservative choice of $\alpha$ (i.e., $4\lambda+2$) than what is suggested above.

\subsection{Extension beyond the 1D unconstrained setting}

So far we have only considered the 1D unconstrained setting. Our results can be extended to higher dimensional domains and bounded domains, which is deferred to Appendix~\ref{section:extension}. 

Most notably, we present an algorithm (Algorithm~\ref{alg:constraint}) for 1D bounded domain: if the domain has diameter $D$, then the \emph{switching cost alone} of this algorithm is bounded by $\tilde O(D\sqrt{\tau})$ on \emph{any time interval of length $\tau$}. Such a property is crucial in \cite{zhang2022adversarial} for the construction of a \emph{strongly adaptive OCO with memory} algorithm. However, the proof in \cite{zhang2022adversarial} critically relies on hard-thresholding a betting fraction, which, as we demonstrated in Section~\ref{subsection:dual_scale}, is suboptimal. In contrast, our new result simultaneously achieves this property and the optimal augmented regret bound. 

\section{LEA with switching cost}\label{section:lea}

Our improved results can also be applied to \emph{LEA with switching cost}, leading to the first comparator adaptive algorithm there. Conversion techniques (from OLO to LEA) without switching costs were studied in \cite{luo2015achieving,orabona2016coin}, and since then, they have become standard tools for the online learning community. Here we present a different view on this conversion problem, based on its connection to the well-known constrained domain reduction \cite{cutkosky2018black} (surveyed in Appendix~\ref{section:extension}). In particular, it leads to a mechanism for incorporating switching costs, with a clear geometric interpretation. 

The setting of \emph{LEA with switching cost} is a special case of the high-dimensional OLO problem (Appendix~\ref{section:extension}). Let $d$ be the number of experts, and we define the domain $\calX$ as the $d$-dimensional probability simplex $\Delta(d)$. Loss gradients $g_t$ satisfy $\norm{g_t}_\infty\leq G$, and switching costs are measured by the $L_1$ norm. The performance metric is still the augmented regret, now defined as
\begin{equation*}
\reg^\lambda_T(u)\defeq \sum_{t=1}^T\inner{g_t}{x_t-u}+\lambda\sum_{t=1}^{T-1}\norm{x_t-x_{t+1}}_1.
\end{equation*}
However, the main difference with OLO is the form of comparator adaptive bounds -- here we aim at $\reg^\lambda_T(u)=O(\sqrt{T\cdot\kl(u||\pi)})$, where $\pi\in\Delta(d)$ is a prior chosen at the beginning. Achieving such a root KL bound relies on techniques different from the OLO setting. 

Existing approaches \cite{luo2015achieving,orabona2016coin} have the following procedure. Given a 1D OLO algorithm that predicts on $\R_+$, independent copies are created for each coordinate and updated using certain surrogate losses. A meta-algorithm queries the coordinate-wise predictions $\{w_{t,i};i\in[1:d]\}$, collects them into a weight vector $w_{t}=[w_{t,1},\ldots,w_{t,d}]$, and finally predicts the scaled weight $x_t=w_t/\norm{w_t}_1$ on $\Delta(d)$. Despite its general success, such an approach has a discontinuity problem when switching costs are incorporated -- if two consecutive weights $w_{t}$ and $w_{t+1}$ are both close to the origin, then simply scaling them to $\Delta(d)$ can lead to a large switching cost, even when $\norm{w_t-w_{t+1}}_1$ is small. This problem is exacerbated by the typical setting\footnote{When $w_t=0$, $x_t$ can be arbitrary on $\Delta(d)$ by definition. However, as $w_t$ changes continuously w.r.t. the observed information, it could hover around $0$ at some point, thus experiencing the sketched problem.} of $w_1=0$, due to the associated analysis. A graphical demonstration is provided in Figure~\ref{fig:projection} (Left).

\begin{figure}[ht]
    \centering
    \includegraphics[width=0.8\textwidth]{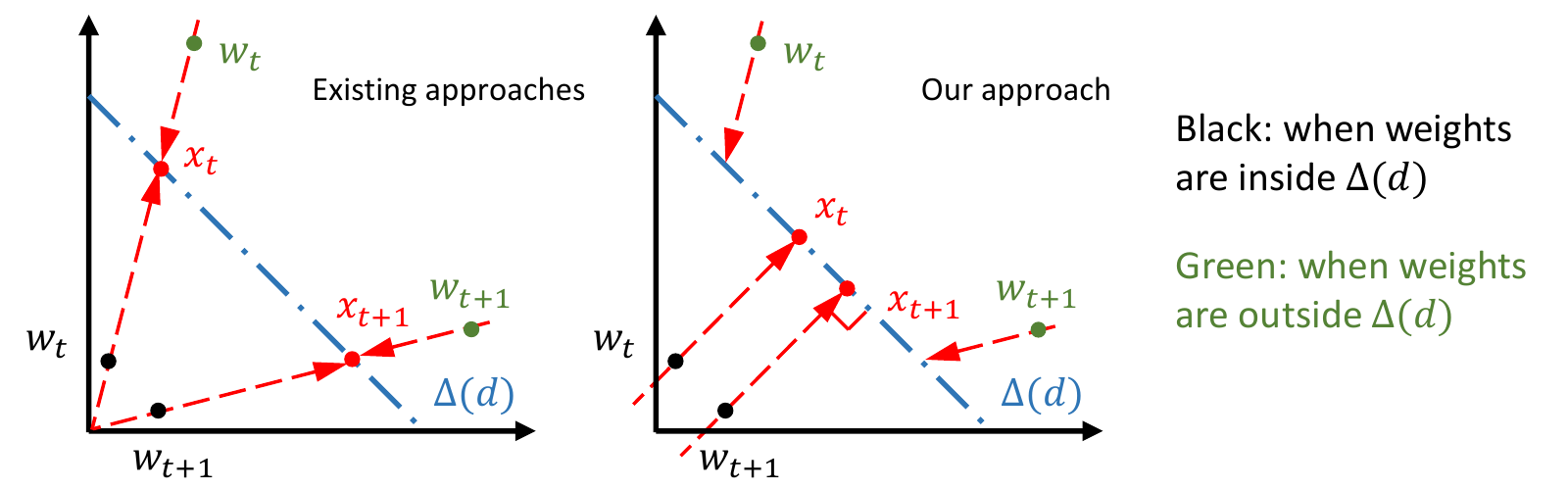}
    \caption{Switching costs in LEA-OLO reductions. Left: existing approaches. Right: ours, where the projection of $w_t$ contains two cases. ($i$) $\norms{w_t}_1\geq 1$, shown in green; ($ii$) $\norms{w_t}_1<1$, shown in black.}
    \label{fig:projection}
\end{figure}

In contrast, our solution is based on \emph{a unified view} of the LEA-OLO reduction and the constrained domain reduction \cite{cutkosky2018black}. Starting without switching costs, we observe that the general Banach version of the latter can also convert OLO to LEA, therefore specialized techniques are not required for this task. Algorithmically, we set $x_t\in\argmin_{x\in\Delta(d)}\norms{x-w_t}_1$ as opposed to $x_t=w_t/\norms{w_t}_1$. The surrogate losses for the base algorithms are also different, which we elaborate in Appendix~\ref{subsection:lea_discussion}. 

A major benefit of this unified view is the non-uniqueness of the $L_1$ norm projection -- if $\norm{w_t}< 1$, then any $x_t\in\Delta(d)$ satisfying $\{x_{t,i}\geq w_{t,i}; \forall i\}$ minimizes $\norm{x-w_t}_1$ on $\Delta(d)$. This brings more flexibility to the algorithm design. Specifically, we adopt
\begin{enumerate}
\item the orthogonal projection $x_t=w_t+d^{-1}(1-\norms{w_t}_1)$ when $\norms{w_t}_1\leq 1$; 
\item the scaling $x_t=w_t/\norms{w_t}_1$ when $\norms{w_t}_1>1$.
\end{enumerate}
The orthogonal projection is better for controlling switching costs, as shown in Figure~\ref{fig:projection} (Right). Concretely, this leads to the first comparator adaptive algorithm for LEA with switching costs. 

\begin{restatable}{theorem}{lea}\label{thm:lea}
For LEA with switching cost, given any prior $\pi$ in the relative interior of $\Delta(d)$, Algorithm~\ref{alg:lea} from Appendix~\ref{subsection:lea_analysis} guarantees
\begin{equation*}
\sum_{t=1}^T\inner{g_t}{x_t-u}+\lambda\sum_{t=1}^{T-1}\norm{x_t-x_{t+1}}_1= \spar{\sqrt{\tv(u||\pi)\cdot \kl(u||\pi)}+1}\cdot O\rpar{\sqrt{(\lambda G+G^2)T}},
\end{equation*}
for all $u\in\Delta(d)$ and $T\in\N_+$. 
\end{restatable}

We emphasize two strengths of this bound.
\begin{enumerate}
\item Since it is comparator adaptive, such a bound only implicitly depends on $d$ through the divergence term $\sqrt{\tv\cdot\kl}$. In favorable cases we may have a good prior $\pi$ such that $\tv(u||\pi)\cdot\kl(u||\pi)=O(1)$; this will save us a $\sqrt{\log d}$ factor compared to minimax algorithms (with switching costs), such as \emph{Follow the Lazy Leader} \cite{kalai2005efficient} and \emph{Shrinking Dartboard} \cite{geulen2010regret}. 
\item Even without switching costs, we improve the $\sqrt{\kl}$ divergence term in existing comparator adaptive bounds \cite{chaudhuri2009parameter,luo2015achieving,orabona2016coin} to $\sqrt{\tv\cdot\kl}$. The latter is better since ($i$) $\tv$ is always less than 1, and ($ii$) there exist $p,q\in\Delta(d)$ such that $\tv(p||q)\cdot\kl(p||q)\leq 1$ but $\kl(p||q)\geq \sqrt{\log d}-o(1)$ (cf. Appendix~\ref{subsection:lea_discussion}). In other words, compared to $\sqrt{\kl}$, the $\sqrt{\tv\cdot\kl}$ bound is never worse (up to constants), and can save at least a $(\log d)^{1/4}$ factor in certain cases. Generalizations of root KL to \emph{f-divergences} have been considered in \cite{alquier2021non,negrea2021minimax}, but to our knowledge, no existing algorithm guarantees a better divergence term than root KL. 
\end{enumerate}

\paragraph{Experiment} We complement our theoretical results by experiments on a portfolio selection task, presented in Appendix~\ref{section:financial}.

\section{Conclusion}

The present work investigates the design of comparator adaptive algorithms in the presence of switching costs. By carefully trading off these two opposite considerations, we propose a simple algorithm for OLO with switching costs, improving the suboptimal bound from our prior work \cite{zhang2022adversarial} to the optimal rate. Notably, the key idea of this algorithm is not guessed, but derived from a continuous-time analysis. Extensions lead to new results for comparator adaptive LEA. 

More generally, through this paper we aim to demonstrate a key strength of the continuous-time PDE analysis -- it makes the generalization of algorithmic structures much easier. Such an observation could open up exciting possibilities. For example, 
\begin{itemize}
\item Does this approach apply to other variants of the online learning problem? 
\item Can we use it to generalize other forms of adaptivity? 
\item Continuous-time potentials have been extensively studied under the framework of \emph{potential theory} \cite{doob1984classical}. Can we borrow techniques from there to further improve the workflow of algorithm design? 
\end{itemize}

\section*{Acknowledgement}
We thank the anonymous NeurIPS reviewers for their feedback. This research was partially supported by the NSF under grants IIS-1914792, DMS-1664644, and CNS-1645681, by the ONR under grants N00014-19-1-2571 and N00014-21-1-2844, by the DOE under grants DE-AR-0001282 and DE-EE0009696, by the NIH under grants R01 GM135930 and UL54 TR004130, and by Boston University. 

\bibliography{Switching}

\newcommand{\etalchar}[1]{$^{#1}$}
\begin{thebibliography}{WWYZ21}

\bibitem[ABH{\etalchar{+}}19]{agarwal2019online}
Naman Agarwal, Brian Bullins, Elad Hazan, Sham Kakade, and Karan Singh.
\newblock Online control with adversarial disturbances.
\newblock In {\em International Conference on Machine Learning}, pages
  111--119. PMLR, 2019.

\bibitem[ABL{\etalchar{+}}13]{andrew2013tale}
Lachlan Andrew, Siddharth Barman, Katrina Ligett, Minghong Lin, Adam Meyerson,
  Alan Roytman, and Adam Wierman.
\newblock A tale of two metrics: Simultaneous bounds on competitiveness and
  regret.
\newblock In {\em Conference on Learning Theory}, pages 741--763. PMLR, 2013.

\bibitem[AHM15]{anava2015online}
Oren Anava, Elad Hazan, and Shie Mannor.
\newblock Online learning for adversaries with memory: price of past mistakes.
\newblock {\em Advances in Neural Information Processing Systems}, 28, 2015.

\bibitem[Alq21]{alquier2021non}
Pierre Alquier.
\newblock Non-exponentially weighted aggregation: Regret bounds for unbounded
  loss functions.
\newblock In {\em International Conference on Machine Learning}, pages
  207--218. PMLR, 2021.

\bibitem[AT18]{altschuler2018online}
Jason Altschuler and Kunal Talwar.
\newblock Online learning over a finite action set with limited switching.
\newblock In {\em Conference On Learning Theory}, pages 1569--1573. PMLR, 2018.

\bibitem[BCKP21]{bhaskara2021power}
Aditya Bhaskara, Ashok Cutkosky, Ravi Kumar, and Manish Purohit.
\newblock Power of hints for online learning with movement costs.
\newblock In {\em International Conference on Artificial Intelligence and
  Statistics}, pages 2818--2826. PMLR, 2021.

\bibitem[BEZ20a]{bayraktar2020finite}
Erhan Bayraktar, Ibrahim Ekren, and Xin Zhang.
\newblock Finite-time 4-expert prediction problem.
\newblock {\em Communications in Partial Differential Equations},
  45(7):714--757, 2020.

\bibitem[BEZ20b]{bayraktar2020asymptotic}
Erhan Bayraktar, Ibrahim Ekren, and Yili Zhang.
\newblock On the asymptotic optimality of the comb strategy for prediction with
  expert advice.
\newblock {\em The Annals of Applied Probability}, 30(6):2517--2546, 2020.

\bibitem[BK99]{blum1999universal}
Avrim Blum and Adam Kalai.
\newblock Universal portfolios with and without transaction costs.
\newblock {\em Machine Learning}, 35(3):193--205, 1999.

\bibitem[BLLS19]{bubeck2019competitively}
S{\'e}bastien Bubeck, Yin~Tat Lee, Yuanzhi Li, and Mark Sellke.
\newblock Competitively chasing convex bodies.
\newblock In {\em Proceedings of the 51st Annual ACM SIGACT Symposium on Theory
  of Computing}, pages 861--868, 2019.

\bibitem[CBDS13]{cesa2013online}
Nicolo Cesa-Bianchi, Ofer Dekel, and Ohad Shamir.
\newblock Online learning with switching costs and other adaptive adversaries.
\newblock {\em Advances in Neural Information Processing Systems}, 26, 2013.

\bibitem[CBL06]{cesa2006prediction}
Nicolo Cesa-Bianchi and G{\'a}bor Lugosi.
\newblock {\em Prediction, learning, and games}.
\newblock Cambridge university press, 2006.

\bibitem[CFH09]{chaudhuri2009parameter}
Kamalika Chaudhuri, Yoav Freund, and Daniel~J Hsu.
\newblock A parameter-free hedging algorithm.
\newblock {\em Advances in neural information processing systems}, 22, 2009.

\bibitem[CGW18]{chen2018smoothed}
Niangjun Chen, Gautam Goel, and Adam Wierman.
\newblock Smoothed online convex optimization in high dimensions via online
  balanced descent.
\newblock In {\em Conference On Learning Theory}, pages 1574--1594. PMLR, 2018.

\bibitem[CLO22]{chen2022better}
Keyi Chen, John Langford, and Francesco Orabona.
\newblock Better parameter-free stochastic optimization with ode updates for
  coin-betting.
\newblock In {\em Proceedings of the AAAI Conference on Artificial
  Intelligence}, pages 6239--6247, 2022.

\bibitem[CLW21]{chen2021impossible}
Liyu Chen, Haipeng Luo, and Chen-Yu Wei.
\newblock Impossible tuning made possible: A new expert algorithm and its
  applications.
\newblock In {\em Conference on Learning Theory}, pages 1216--1259. PMLR, 2021.

\bibitem[CO96]{cover1996universal}
Thomas~M Cover and Erik Ordentlich.
\newblock Universal portfolios with side information.
\newblock {\em IEEE Transactions on Information Theory}, 42(2):348--363, 1996.

\bibitem[CO18]{cutkosky2018black}
Ashok Cutkosky and Francesco Orabona.
\newblock Black-box reductions for parameter-free online learning in banach
  spaces.
\newblock In {\em Conference On Learning Theory}, pages 1493--1529. PMLR, 2018.

\bibitem[Cov91]{cover1991universal}
Thomas~M Cover.
\newblock Universal portfolios.
\newblock {\em Mathematical Finance}, 1(1):1--29, 1991.

\bibitem[Cut19]{cutkosky2019combining}
Ashok Cutkosky.
\newblock Combining online learning guarantees.
\newblock In {\em Conference on Learning Theory}, pages 895--913. PMLR, 2019.

\bibitem[Cut20]{cutkosky2020parameter}
Ashok Cutkosky.
\newblock Parameter-free, dynamic, and strongly-adaptive online learning.
\newblock In {\em International Conference on Machine Learning}, pages
  2250--2259. PMLR, 2020.

\bibitem[CV10]{chernov2010prediction}
Alexey Chernov and Vladimir Vovk.
\newblock Prediction with advice of unknown number of experts.
\newblock In {\em Proceedings of the Twenty-Sixth Conference on Uncertainty in
  Artificial Intelligence}, pages 117--125, 2010.

\bibitem[CYLK20]{chen2020minimax}
Lin Chen, Qian Yu, Hannah Lawrence, and Amin Karbasi.
\newblock Minimax regret of switching-constrained online convex optimization:
  No phase transition.
\newblock {\em Advances in Neural Information Processing Systems},
  33:3477--3486, 2020.

\bibitem[DGSS15]{daniely2015strongly}
Amit Daniely, Alon Gonen, and Shai Shalev-Shwartz.
\newblock Strongly adaptive online learning.
\newblock In {\em International Conference on Machine Learning}, pages
  1405--1411. PMLR, 2015.

\bibitem[DHS11]{duchi2011adaptive}
John Duchi, Elad Hazan, and Yoram Singer.
\newblock Adaptive subgradient methods for online learning and stochastic
  optimization.
\newblock {\em Journal of machine learning research}, 12(7), 2011.

\bibitem[DK20]{drenska2020prediction}
Nadejda Drenska and Robert~V Kohn.
\newblock Prediction with expert advice: A {PDE} perspective.
\newblock {\em Journal of Nonlinear Science}, 30(1):137--173, 2020.

\bibitem[DM19]{daniely2019competitive}
Amit Daniely and Yishay Mansour.
\newblock Competitive ratio vs regret minimization: achieving the best of both
  worlds.
\newblock In {\em Algorithmic Learning Theory}, pages 333--368. PMLR, 2019.

\bibitem[Doc16]{dochow2016online}
Robert Dochow.
\newblock {\em Online algorithms for the portfolio selection problem}.
\newblock Springer, 2016.

\bibitem[Doo84]{doob1984classical}
Joseph~L Doob.
\newblock {\em Classical potential theory and its probabilistic counterpart}.
\newblock Springer, 1984.

\bibitem[FRS15]{foster2015adaptive}
Dylan~J Foster, Alexander Rakhlin, and Karthik Sridharan.
\newblock Adaptive online learning.
\newblock {\em Advances in Neural Information Processing Systems},
  28:3375--3383, 2015.

\bibitem[FRS18]{foster2018online}
Dylan~J Foster, Alexander Rakhlin, and Karthik Sridharan.
\newblock Online learning: Sufficient statistics and the {B}urkholder method.
\newblock In {\em Conference On Learning Theory}, pages 3028--3064. PMLR, 2018.

\bibitem[Fuj08]{fujita2008random}
Takahiko Fujita.
\newblock A random walk analogue of {L}{\'e}vy’s theorem.
\newblock {\em Studia scientiarum mathematicarum Hungarica}, 45(2):223--233,
  2008.

\bibitem[GHP22]{greenstreet2022efficient}
Laura Greenstreet, Nicholas~JA Harvey, and Victor~Sanches Portella.
\newblock Efficient and optimal fixed-time regret with two experts.
\newblock In {\em International Conference on Algorithmic Learning Theory},
  pages 436--464. PMLR, 2022.

\bibitem[GLSW19]{goel2019beyond}
Gautam Goel, Yiheng Lin, Haoyuan Sun, and Adam Wierman.
\newblock Beyond online balanced descent: An optimal algorithm for smoothed
  online optimization.
\newblock {\em Advances in Neural Information Processing Systems}, 32, 2019.

\bibitem[Gof14]{gofer2014higher}
Eyal Gofer.
\newblock Higher-order regret bounds with switching costs.
\newblock In {\em Conference on Learning Theory}, pages 210--243. PMLR, 2014.

\bibitem[GVW10]{geulen2010regret}
Sascha Geulen, Berthold V{\"o}cking, and Melanie Winkler.
\newblock Regret minimization for online buffering problems using the weighted
  majority algorithm.
\newblock In {\em Conference on Learning Theory}, pages 132--143, 2010.

\bibitem[Haz16]{hazan2016introduction}
Elad Hazan.
\newblock Introduction to online convex optimization.
\newblock {\em Foundations and Trends{\textregistered} in Optimization},
  2(3-4):157--325, 2016.

\bibitem[HLPR20]{harvey2020optimal}
Nicholas~JA Harvey, Christopher Liaw, Edwin~A Perkins, and Sikander Randhawa.
\newblock Optimal anytime regret for two experts.
\newblock In {\em 2020 IEEE 61st Annual Symposium on Foundations of Computer
  Science (FOCS)}, pages 1404--1415. IEEE, 2020.

\bibitem[HSSW98]{helmbold1998line}
David~P Helmbold, Robert~E Schapire, Yoram Singer, and Manfred~K Warmuth.
\newblock On-line portfolio selection using multiplicative updates.
\newblock {\em Mathematical Finance}, 8(4):325--347, 1998.

\bibitem[JC22]{jacobsen2022parameter}
Andrew Jacobsen and Ashok Cutkosky.
\newblock Parameter-free mirror descent.
\newblock In {\em Conference on Learning Theory}, pages 4160--4211. PMLR, 2022.

\bibitem[JO19]{jun2019parameter}
Kwang-Sung Jun and Francesco Orabona.
\newblock Parameter-free locally differentially private stochastic subgradient
  descent.
\newblock {\em arXiv preprint arXiv:1911.09564}, 2019.

\bibitem[KKW20]{kobzar2020a_new}
Vladimir~A Kobzar, Robert~V Kohn, and Zhilei Wang.
\newblock New potential-based bounds for prediction with expert advice.
\newblock In {\em Conference on Learning Theory}, pages 2370--2405. PMLR, 2020.

\bibitem[Kle13]{klenke2013probability}
Achim Klenke.
\newblock {\em Probability theory: a comprehensive course}.
\newblock Springer Science \& Business Media, 2013.

\bibitem[KS10]{kohn2010deterministic}
Robert~V Kohn and Sylvia Serfaty.
\newblock A deterministic-control-based approach to fully nonlinear parabolic
  and elliptic equations.
\newblock {\em Communications on pure and applied mathematics},
  63(10):1298--1350, 2010.

\bibitem[Kud82]{kudvzma1982ito}
R~Kud{\v{z}}ma.
\newblock Ito's formula for a random walk.
\newblock {\em Lithuanian Mathematical Journal}, 22(3):302--306, 1982.

\bibitem[KV02]{kalai2002efficient}
Adam~Tauman Kalai and Santosh Vempala.
\newblock Efficient algorithms for universal portfolios.
\newblock {\em Journal of Machine Learning Research}, pages 423--440, 2002.

\bibitem[KV05]{kalai2005efficient}
Adam Kalai and Santosh Vempala.
\newblock Efficient algorithms for online decision problems.
\newblock {\em Journal of Computer and System Sciences}, 71(3):291--307, 2005.

\bibitem[KVE15]{koolen2015second}
Wouter~M Koolen and Tim Van~Erven.
\newblock Second-order quantile methods for experts and combinatorial games.
\newblock In {\em Conference on Learning Theory}, pages 1155--1175. PMLR, 2015.

\bibitem[LH14]{li2014online}
Bin Li and Steven~CH Hoi.
\newblock Online portfolio selection: A survey.
\newblock {\em ACM Computing Surveys (CSUR)}, 46(3):1--36, 2014.

\bibitem[LQL20]{li2020online}
Yingying Li, Guannan Qu, and Na~Li.
\newblock Online optimization with predictions and switching costs: Fast
  algorithms and the fundamental limit.
\newblock {\em IEEE Transactions on Automatic Control}, 66(10):4761--4768,
  2020.

\bibitem[LS15]{luo2015achieving}
Haipeng Luo and Robert~E Schapire.
\newblock Achieving all with no parameters: Adanormalhedge.
\newblock In {\em Conference on Learning Theory}, pages 1286--1304. PMLR, 2015.

\bibitem[LWZ18]{luo2018efficient}
Haipeng Luo, Chen-Yu Wei, and Kai Zheng.
\newblock Efficient online portfolio with logarithmic regret.
\newblock {\em Advances in Neural Information Processing Systems}, 31, 2018.

\bibitem[MK20]{mhammedi2020lipschitz}
Zakaria Mhammedi and Wouter~M Koolen.
\newblock Lipschitz and comparator-norm adaptivity in online learning.
\newblock In {\em Conference on Learning Theory}, pages 2858--2887. PMLR, 2020.

\bibitem[MO14]{mcmahan2014unconstrained}
H~Brendan McMahan and Francesco Orabona.
\newblock Unconstrained online linear learning in hilbert spaces: Minimax
  algorithms and normal approximations.
\newblock In {\em Conference on Learning Theory}, pages 1020--1039. PMLR, 2014.

\bibitem[MR22]{mhammedi2022damped}
Zakaria Mhammedi and Alexander Rakhlin.
\newblock Damped {O}nline {N}ewton {S}tep for portfolio selection.
\newblock In {\em Conference on Learning Theory}, pages 5561--5595. PMLR, 2022.

\bibitem[NBC{\etalchar{+}}21]{negrea2021minimax}
Jeffrey Negrea, Blair Bilodeau, Nicol{\`o} Campolongo, Francesco Orabona, and
  Dan Roy.
\newblock Minimax optimal quantile and semi-adversarial regret via
  root-logarithmic regularizers.
\newblock {\em Advances in Neural Information Processing Systems}, 34, 2021.

\bibitem[OLL17]{orseau2017soft}
Laurent Orseau, Tor Lattimore, and Shane Legg.
\newblock Soft-bayes: Prod for mixtures of experts with log-loss.
\newblock In {\em International Conference on Algorithmic Learning Theory},
  pages 372--399. PMLR, 2017.

\bibitem[OP16]{orabona2016coin}
Francesco Orabona and D{\'a}vid P{\'a}l.
\newblock Coin betting and parameter-free online learning.
\newblock {\em Advances in Neural Information Processing Systems}, 29, 2016.

\bibitem[Ora13]{orabona2013dimension}
Francesco Orabona.
\newblock Dimension-free exponentiated gradient.
\newblock {\em Advances in Neural Information Processing Systems}, 26, 2013.

\bibitem[Ora19]{orabona2019modern}
Francesco Orabona.
\newblock A modern introduction to online learning.
\newblock {\em arXiv preprint arXiv:1912.13213}, 2019.

\bibitem[OT17]{orabona2017training}
Francesco Orabona and Tatiana Tommasi.
\newblock Training deep networks without learning rates through coin betting.
\newblock {\em Advances in Neural Information Processing Systems},
  30:2160--2170, 2017.

\bibitem[PLH22]{portella2022continuous}
Victor~Sanches Portella, Christopher Liaw, and Nicholas~JA Harvey.
\newblock Continuous prediction with experts' advice.
\newblock {\em arXiv preprint arXiv:2206.00236}, 2022.

\bibitem[Rok17]{rokhlin2017pde}
Dmitry~B Rokhlin.
\newblock {PDE} approach to the problem of online prediction with expert
  advice: a construction of potential-based strategies.
\newblock {\em arXiv preprint arXiv:1705.01091}, 2017.

\bibitem[Sel20]{sellke2020chasing}
Mark Sellke.
\newblock Chasing convex bodies optimally.
\newblock In {\em Proceedings of the Fourteenth Annual ACM-SIAM Symposium on
  Discrete Algorithms}, pages 1509--1518. SIAM, 2020.

\bibitem[Sim20]{simchowitz2020making}
Max Simchowitz.
\newblock Making non-stochastic control (almost) as easy as stochastic.
\newblock {\em Advances in Neural Information Processing Systems},
  33:18318--18329, 2020.

\bibitem[SK21]{sherman2021lazy}
Uri Sherman and Tomer Koren.
\newblock Lazy {OCO}: Online convex optimization on a switching budget.
\newblock In {\em Conference on Learning Theory}, pages 3972--3988. PMLR, 2021.

\bibitem[SM12]{streeter2012no}
Matthew Streeter and Brendan Mcmahan.
\newblock No-regret algorithms for unconstrained online convex optimization.
\newblock {\em Advances in Neural Information Processing Systems}, 25, 2012.

\bibitem[SS11]{shalev2011online}
Shai Shalev-Shwartz.
\newblock Online learning and online convex optimization.
\newblock {\em Foundations and trends in Machine Learning}, 4(2):107--194,
  2011.

\bibitem[SSH20]{simchowitz2020improper}
Max Simchowitz, Karan Singh, and Elad Hazan.
\newblock Improper learning for non-stochastic control.
\newblock In {\em Conference on Learning Theory}, pages 3320--3436. PMLR, 2020.

\bibitem[vdH19]{van2019user}
Dirk van~der Hoeven.
\newblock User-specified local differential privacy in unconstrained adaptive
  online learning.
\newblock {\em Advances in Neural Information Processing Systems}, 32, 2019.

\bibitem[WWYZ21]{wang2021online}
Guanghui Wang, Yuanyu Wan, Tianbao Yang, and Lijun Zhang.
\newblock Online convex optimization with continuous switching constraint.
\newblock {\em Advances in Neural Information Processing Systems}, 34, 2021.

\bibitem[ZAK22]{zimmert2022pushing}
Julian Zimmert, Naman Agarwal, and Satyen Kale.
\newblock Pushing the efficiency-regret pareto frontier for online learning of
  portfolios and quantum states.
\newblock In {\em Conference on Learning Theory}, pages 182--226. PMLR, 2022.

\bibitem[ZCP22a]{zhang2022adversarial}
Zhiyu Zhang, Ashok Cutkosky, and Ioannis Paschalidis.
\newblock Adversarial tracking control via strongly adaptive online learning
  with memory.
\newblock In {\em International Conference on Artificial Intelligence and
  Statistics}, pages 8458--8492. PMLR, 2022.

\bibitem[ZCP22b]{zhang2022pde}
Zhiyu Zhang, Ashok Cutkosky, and Ioannis Paschalidis.
\newblock {PDE}-based optimal strategy for unconstrained online learning.
\newblock In {\em International Conference on Machine Learning}, pages
  26085--26115. PMLR, 2022.

\bibitem[Zhu14]{zhu2014two}
Kangping Zhu.
\newblock {\em Two problems in applications of PDE}.
\newblock PhD thesis, New York University, 2014.

\bibitem[Zin03]{zinkevich2003online}
Martin Zinkevich.
\newblock Online convex programming and generalized infinitesimal gradient
  ascent.
\newblock In {\em Proceedings of the 20th International Conference on Machine
  Learning}, pages 928--936, 2003.

\bibitem[ZJLY21]{zhang2021revisiting}
Lijun Zhang, Wei Jiang, Shiyin Lu, and Tianbao Yang.
\newblock Revisiting smoothed online learning.
\newblock {\em Advances in Neural Information Processing Systems}, 34, 2021.

\end{thebibliography}

\newpage
\section*{Appendix}
\appendix

\paragraph{Organization} Appendix~\ref{section:detail_olo} presents omitted details of our 1D OLO results. Section~\ref{section:extension} extends such results to more general OLO settings. Section~\ref{section:detail_lea} presents details on LEA. Section~\ref{section:financial} numerically tests our approach in a portfolio selection problem. 

\section{Details on 1D OLO}\label{section:detail_olo}

This section presents detailed discussions and omitted proofs for our 1D unconstrained OLO algorithm. 

\subsection[]{The suboptimal algorithm from \cite{zhang2022adversarial}}\label{subsection:suboptimal}

Let us start by summarizing the existing result from our prior work \cite[Algorithm~1]{zhang2022adversarial}, which is the first comparator adaptive algorithm for 1D unconstrained OLO with switching costs. The original version in \cite{zhang2022adversarial} considers a bounded domain and an extra regularization term, which are removed below for a clear comparison. $\Pi$ denotes the absolute value projection. 

\begin{algorithm*}[ht]
\caption{The suboptimal algorithm from \cite{zhang2022adversarial}\label{algorithm:1d}}
\begin{algorithmic}[1]
\REQUIRE A hyperparameter $C>0$, the Lipschitz constant $G$, and the switching cost weight $\lambda$.
\STATE Define $K=G+\lambda$. Initialize internal variables as $\wel_0=C\cdot K$, and $\beta_1,x_1=0$. 
\FOR{$t=1,2,\ldots$}
\STATE Make a prediction $x_t$, observe a loss gradient $g_t$. 
\STATE Define an \emph{unprojected} betting fraction as $\hat\beta_{t+1}=-\sum_{i=1}^t g_i/(2K^2t)$. 
\STATE Define a hard threshold for the betting fraction, $\B_{t+1}=[-1/(K\sqrt{2t}),1/(K\sqrt{2t})]$.
\STATE Update the \emph{projected} betting fraction as $\beta_{t+1}=\Pi_{\B_{t+1}}(\hat\beta_{t+1})$.
\STATE Assign $\wel_t$ as the solution to the following equation (uniqueness can be proved),
\begin{equation}
\wel_t=(1- g_t\beta_t)\wel_{t-1}-\lambda|\beta_t\wel_{t-1}-\beta_{t+1}\wel_t|.\label{eq:wel_update}
\end{equation}

\STATE Calculate the next prediction, $x_{t+1}=\beta_{t+1}\wel_{t}$.
\ENDFOR
\end{algorithmic}
\end{algorithm*}

Both the algorithm and its analysis are analogous to \cite{cutkosky2018black}, which recasts the selection of betting fractions as an ``inner'' online learning problem. Nonetheless, incorporating switching costs requires extra components (Line 5-7), making the whole analysis nontrivial. 

\begin{theorem}[Theorem~1 of \cite{zhang2022adversarial}, adapted]\label{theorem:1d}
Algorithm~\ref{algorithm:1d} guarantees for all $T\in\N_+$ and $u\in\R$,
\begin{equation*}
\reg^\lambda_T(u)\leq (G+\lambda)\spar{C+\abs{u}\sqrt{2T}\rpar{\frac{3}{2}+\log\frac{\sqrt{2}\abs{u}T^{5/2}}{C}}}.
\end{equation*}
\end{theorem}

For now, let us only consider the dependence on $u$ and $T$. Compared to typical results on comparator adaptivity, the above bound has two limitations. First, the bound does not achieve the optimal loss-regret trade-off \cite{zhang2022adversarial} -- the constraint $\reg^\lambda_T(0)\leq O(1)$ on the \emph{cumulative loss} of the algorithm is too harsh, such that the leading regret term ($\reg^\lambda_T(u)$ with large $\abs{u}$) suffers a logarithmic penalty on $T$ (relative to the usual $O(\sqrt{T})$ minimax rate). Second, even if we only consider this particular loss-regret trade-off, i.e., $\reg^\lambda_T(0)\leq O(1)$, the logarithmic terms are not optimal (being outside the square root). In other words, the bound is not Pareto-optimal. The present paper simultaneously improves these two suboptimalities. 

On a separate note, let us consider its dependence on $G$ and $\lambda$, which is more subtle.\footnote{We thank the NeurIPS reviewer KR3f for insightful comments.} In its vanilla form, the above bound has the leading term $\tilde O(\max\{G,\lambda\}\abs{u}\sqrt{T})$, but we can run a meta-algorithm \cite[Algorithm~3]{zhang2022adversarial} on the top to improve it to $\tilde O\rpar{\abs{u}\sqrt{\max\{G,\lambda\}\sum_{t=1}^T\abs{g_t}}}$. The pseudo-code is presented as Algorithm~\ref{algorithm:meta}. Its main idea is to adaptively slow down the update of the base algorithm, depending on the observed gradients. 

\begin{algorithm*}[ht]
\caption{Meta-algorithm \cite[Algorithm~3]{zhang2022adversarial}, adapted\label{algorithm:meta}}
\begin{algorithmic}[1]
\REQUIRE The Lipschitz constant $G$ and the switching cost weight $\lambda$.
\STATE Initialize a base algorithm $\A$ as a copy of Algorithm~\ref{algorithm:1d}.
\STATE Initialize $i=1$ and an accumulator $Z_i=0$. Query the first output of $\A$ and assign it to $w_i$. 
\FOR{$t=1,2,\ldots$}
\STATE Predict $x_t\leftarrow w_i$, observe $g_t$, let $Z_i\leftarrow Z_i+g_t$. 
\IF{$\norms{Z_i}> \max\{\lambda,G\}$}\label{line:ada_threshold}
\STATE Send $Z_i$ to $\A$ as the $i$-th loss. Let $i\leftarrow i+1$. 
\STATE Set $Z_i=0$. Query the $i$-th output of $\A$ and assign it to $w_i$. 
\ENDIF
\ENDFOR
\end{algorithmic}
\end{algorithm*}

\begin{theorem}[Theorem~6 of \cite{zhang2022adversarial}, adapted]
Algorithm~\ref{algorithm:meta} guarantees for all $T\in\N_+$ and $u\in\R$,
\begin{equation*}
\reg^\lambda_T(u)\leq (G+\lambda)C+\abs{u}\tilde O\rpar{\max\{G,\lambda\}+\sqrt{\max\{G,\lambda\}\sum_{t=1}^T\abs{g_t}}}.
\end{equation*}
\end{theorem}

When $\lambda$ is large, the $O(\sqrt{\lambda})$ rate on the leading term is optimal. Moreover, in the presence of switching costs, the gradient adaptivity term $\sqrt{\sum_{t=1}^T\abs{g_t}}$ is a strong one, since second-order gradient adaptivity $\sqrt{\sum_{t=1}^T\abs{g_t}^2}$ is not achievable \cite{gofer2014higher}. Note that we can also run this meta-algorithm on top of our improved base algorithm (Algorithm~\ref{alg:1d}), such that the latter achieves gradient adaptivity as well. Due to this reason, when comparing the results of the present work to \cite{zhang2022adversarial}, we mostly leave the dependence on $\lambda$ and the gradient adaptivity out of the comparison. 

\subsection{A few basic lemmas}

Before proving our main result (Theorem~\ref{thm:main}), we present a few basic lemmas on Algorithm~\ref{alg:1d} and the potential function $V_\alpha$ (\ref{eq:alpha}), which will be useful later on. The first lemma shows the monotonicity of the discrete derivative strategy, which is quite intuitive. 

\begin{lemma}\label{lemma:mono}
If the potential $V(t,S)$ is even and convex in $S$, then $\bar\nabla_SV(t,S)$ is odd and monotonically increasing in $S$. 
\end{lemma}

\begin{proof}[Proof of Lemma~\ref{lemma:mono}]
$\bar\nabla_SV(t,S)$ is odd due to the simple relation
\begin{equation*}
\bar \nabla_{S} V(t,-S)=\frac{1}{2}\spar{V(t,-S+1)-V(t,-S-1)}=\frac{1}{2}\spar{V(t,S-1)-V(t,S+1)}=-\bar \nabla_{S} V(t,S).
\end{equation*}

As for the monotonicity, it is equivalent to showing for all $\delta\geq 0$, 
\begin{equation*}
V(t,S+1+\delta)-V(S-1+\delta)\geq V(t,S+1)-V(S-1).
\end{equation*}
This follows from the convexity of $V(t,\cdot)$, as
\begin{equation*}
V(t,S+1)\leq \frac{2}{2+\delta}V(t,S+1+\delta)+\frac{\delta}{2+\delta}V(t,S-1),
\end{equation*}
\begin{equation*}
V(t,S-1+\delta)\leq \frac{\delta}{2+\delta}V(t,S+1+\delta)+\frac{2}{2+\delta}V(t,S-1).\qedhere
\end{equation*}
\end{proof}

Now, for the potential function $V_\alpha$, we compute its continuous partial derivatives. The proof is straightforward calculation, therefore omitted. 

\begin{lemma}\label{lemma:derivs}
For any $\alpha>0$, $V_\alpha$ defined in (\ref{eq:alpha}) is even and convex. Moreover, 
\vspace{-20pt}
\begin{multicols}{2}
\begin{equation*}
\nabla_{S}V_{\alpha}(t,S)=C\int_{0}^{S/\sqrt{4\alpha t}}\exp\rpar{x^2}dx,
\end{equation*}
\begin{equation*}
\nabla_{SS}V_{\alpha}(t,S)=\frac{C}{2\sqrt{\alpha t}}\exp\rpar{\frac{S^2}{4\alpha t}},
\end{equation*}

\begin{equation*}
\nabla_{SSS}V_{\alpha}(t,S)=\frac{CS}{4(\alpha t)^{3/2}}\exp\rpar{\frac{S^2}{4\alpha t}},
\end{equation*}
\begin{equation*}
\nabla_{t}V_{\alpha}(t,S)=-\frac{C\sqrt{\alpha}}{2\sqrt{t}}\exp\rpar{\frac{S^2}{4\alpha t}}.
\end{equation*}
\end{multicols}
\end{lemma}

Based on the above, the discrete derivative $\bar\nabla_S V_\alpha$ has the following properties. 

\begin{lemma}\label{lemma:discrete_deriv}
For all $\alpha>0$, $t\geq 0$ and $S\geq 0$, 
\begin{enumerate}
\item $\bar\nabla_S V_\alpha(t,S)$ as a function of $t$ is decreasing and convex;
\item $\bar\nabla_S V_\alpha(t,S)$ as a function of $S$ is convex. 
\end{enumerate}
\end{lemma}

\begin{proof}[Proof of Lemma~\ref{lemma:discrete_deriv}]
Considering the first part of the lemma, 
\begin{equation*}
\nabla_t\spar{\bar\nabla_SV_\alpha(t,S)}=\frac{1}{2}\spar{\nabla_tV_\alpha(t,S+1)-\nabla_tV_\alpha(t,S-1)}=-\frac{C\sqrt{\alpha}}{4\sqrt{t}}\exp\rpar{\frac{S^2+1}{4\alpha t}}\sinh\rpar{\frac{S}{2\alpha t}},
\end{equation*}
which, when $S\geq 0$, is negative and increasing in $t$. Therefore, $\bar\nabla_SV_{\alpha}(t,S)$ as a function of $t$ is decreasing and convex. Similarly, 
\begin{equation*}
\nabla_S\spar{\bar\nabla_SV_{\alpha}(t,S)}=\frac{1}{2}\spar{\nabla_S V_{\alpha}(t,S+1)-\nabla_S V_{\alpha}(t,S-1)}=\frac{C}{2}\int_{(S-1)/\sqrt{4\alpha t}}^{(S+1)/\sqrt{4\alpha t}}\exp\rpar{x^2}dx,
\end{equation*}
which is increasing in $S$. Therefore, $\bar\nabla_SV_{\alpha}(t,S)$ as a function of $S$ is convex.
\end{proof}

\subsection{Proof of Theorem~\ref{thm:main}}\label{subsection:proof_main}

In this subsection, we prove Theorem~\ref{thm:main}, the regret bound of our 1D OLO algorithm with switching costs. As sketched in Section~\ref{subsection:sketch}, our proof relies on two important lemmas, Lemma~\ref{lemma:switch} and \ref{lemma:vanish}. We prove them first. 

\switch*

\begin{proof}[Proof of Lemma~\ref{lemma:switch}]
First, since $\bar \nabla_S V_\alpha(t,S)$ is monotonic in $S$ due to Lemma~\ref{lemma:mono}, we have
\begin{equation*}
\abs{x_t-x_{t+1}}=\abs{\bar\nabla_S V_\alpha(t,S_{t-1})-\bar\nabla_S V_\alpha(t+1,S_{t})}\leq \max_{c=\pm 1}\abs{\bar\nabla_S V_\alpha(t,S_{t-1})-\bar\nabla_S V_\alpha(t+1,S_{t-1}+c)}.
\end{equation*}

For clarity, from the RHS we define
\begin{equation*}
f(t,S)\defeq\max_{c=\pm 1}\abs{\bar\nabla_S V_\alpha(t,S)-\bar\nabla_S V_\alpha(t+1,S+c)}. 
\end{equation*}
It is even in $S$, as
\begin{align*}
f(t,-S)&=\max_{c=\pm 1}\abs{\bar\nabla_S V_\alpha(t,-S)-\bar\nabla_S V_\alpha(t+1,-S+c)}\\
&=\max_{c=\pm 1}\abs{-\bar\nabla_S V_\alpha(t,S)+\bar\nabla_S V_\alpha(t+1,S-c)}\tag{Lemma~\ref{lemma:mono}}\\
&=\max_{c=\pm 1}\abs{\bar\nabla_S V_\alpha(t,S)-\bar\nabla_S V_\alpha(t+1,S-c)}=f(t,S).
\end{align*}
Therefore, we can restrict the rest of the proof to $S\geq 0$, and the remaining task is to upper bound $f(t,S)$ for all $0\leq S\leq t-1$.

From Lemma~\ref{lemma:mono} and \ref{lemma:discrete_deriv}, 
\begin{equation*}
\bar\nabla_S V_{\alpha}(t+1,S-1)\leq \bar\nabla_S V_{\alpha}(t+1,S)\leq \bar\nabla_S V_{\alpha}(t,S), 
\end{equation*}
\begin{equation*}
\bar\nabla_S V_{\alpha}(t+1,S-1)\leq \bar\nabla_S V_{\alpha}(t+1,S+1).
\end{equation*}
Therefore, if $\bar\nabla_S V_{\alpha}(t+1,S-1)\leq \bar\nabla_S V_{\alpha}(t,S)\leq \bar\nabla_S V_{\alpha}(t+1,S+1)$, then
\begin{align*}
f(t,S)&=\max\left\{\abs{\bar\nabla_S V_{\alpha}(t,S)-\bar\nabla_S V_{\alpha}(t+1,S-1)}, \abs{\bar\nabla_S V_{\alpha}(t,S)-\bar\nabla_S V_{\alpha}(t+1,S+1)}\right\}\\
&\leq \bar\nabla_S V_{\alpha}(t+1,S+1)-\bar\nabla_S V_{\alpha}(t+1,S-1). 
\end{align*}
On the other hand, if $\bar\nabla_S V_{\alpha}(t+1,S+1)\leq \bar\nabla_S V_{\alpha}(t,S)$, then
\begin{equation*}
f(t,S)=\bar\nabla_S V_{\alpha}(t,S)-\bar\nabla_S V_{\alpha}(t+1,S-1).
\end{equation*}
Combining the above, 
\begin{equation*}
f(t,S)\leq \max\left\{\bar\nabla_S V_{\alpha}(t,S)-\bar\nabla_S V_{\alpha}(t+1,S-1), \bar\nabla_S V_{\alpha}(t+1,S+1)-\bar\nabla_S V_{\alpha}(t+1,S-1)\right\}. 
\end{equation*}

Our goal next is to upper bound $f(t,S)$ by $\bar\nabla_S V_{\alpha}(t,S+1)-\bar\nabla_S V_{\alpha}(t,S-1)$, which can be divided into two cases. 

\paragraph{Case 1} We aim to show that
\begin{equation*}
\bar\nabla_S V_\alpha(t,S)-\bar\nabla_S V_\alpha(t+1,S-1)\leq \bar\nabla_S V_\alpha(t,S+1)-\bar\nabla_S V_\alpha(t,S-1),
\end{equation*}
which is equivalent to
\begin{equation}\label{eq:caseone}
\bar\nabla_S V_\alpha(t,S-1)-\bar\nabla_S V_\alpha(t+1,S-1)\leq \bar\nabla_S V_\alpha(t,S+1)-\bar\nabla_S V_\alpha(t,S).
\end{equation}
Note that this trivially holds when $0\leq S<1$: due to Lemma~\ref{lemma:discrete_deriv}, the RHS is always positive; however, the LHS is negative due to $\bar\nabla_S V_\alpha(t,S-1)$ being increasing in $t$ (Lemma~\ref{lemma:mono} and \ref{lemma:discrete_deriv} Part 1). Therefore, we only need to show (\ref{eq:caseone}) for all $S\geq 1$. 

To this end, with $S\geq 1$, we apply the convexity of $\bar\nabla_S V_\alpha$ from Lemma~\ref{lemma:discrete_deriv}:
\begin{equation*}
\bar\nabla_S V_\alpha(t,S+1)-\bar\nabla_S V_\alpha(t,S)\geq \nabla_S\spar{\bar\nabla_S V_\alpha(t,S)},
\end{equation*}
\begin{equation*}
\bar\nabla_S V_\alpha(t,S-1)-\bar\nabla_S V_\alpha(t+1,S-1)\leq -\nabla_t\spar{\bar\nabla_S V_\alpha(t,S-1)}. 
\end{equation*}
Consequently, it suffices to show that
\begin{equation*}
-\nabla_t\spar{\bar\nabla_S V_\alpha(t,S-1)}\leq \nabla_S\spar{\bar\nabla_S V_\alpha(t,S)}.
\end{equation*}

Now it is time to invoke the specific form of $V_\alpha$. We may reuse $\nabla_S\spar{\bar\nabla_SV_\alpha(t,S)}$ and $\nabla_t\spar{\bar\nabla_SV_\alpha(t,S)}$ calculated from the proof of Lemma~\ref{lemma:discrete_deriv}. 
\begin{equation*}
\nabla_S\spar{\bar\nabla_S V_\alpha(t,S)}=\frac{C}{2}\int_{(S-1)/\sqrt{4\alpha t}}^{(S+1)/\sqrt{4\alpha t}}\exp\rpar{x^2}dx\geq \frac{C}{2\sqrt{\alpha t}}\exp\rpar{\frac{S^2}{4\alpha t}},
\end{equation*}
and for all $1\leq S\leq t-1$, 
\begin{align*}
-\nabla_t\spar{\bar\nabla_S V_\alpha(t,S-1)}&=\frac{C\sqrt{\alpha}}{4\sqrt{t}}\exp\rpar{\frac{(S-1)^2+1}{4\alpha t}}\sinh\rpar{\frac{S-1}{2\alpha t}}\\
&=\frac{C\sqrt{\alpha}}{8\sqrt{t}}\exp\rpar{\frac{S^2}{4\alpha t}}\spar{1-\exp\rpar{\frac{-S+1}{\alpha t}}}\\
&\leq \frac{C\sqrt{\alpha}}{8\sqrt{t}}\exp\rpar{\frac{S^2}{4\alpha t}}\spar{1-\exp\rpar{-\frac{1}{\alpha}}}\tag{$S-1\leq t$}\\
&\leq \frac{C}{8\sqrt{\alpha t}}\exp\rpar{\frac{S^2}{4\alpha t}}.\tag{$\exp(x)\geq x+1$}
\end{align*}
Therefore, $-\nabla_t\spar{\bar\nabla_S V_\alpha(t,S-1)}\leq \nabla_S\spar{\bar\nabla_S V_\alpha(t,S)}$, which proves (\ref{eq:caseone}) and concludes Case 1. 

\paragraph{Case 2} We aim to show that
\begin{equation*}
\bar\nabla_S V_\alpha(t+1,S+1)-\bar\nabla_S V_\alpha(t+1,S-1)\leq \bar\nabla_S V_\alpha(t,S+1)-\bar\nabla_S V_\alpha(t,S-1).
\end{equation*}
This is straightforward, as
\begin{align*}
\nabla_t\spar{\bar\nabla_S V_\alpha(t,S+1)-\bar\nabla_S V_\alpha(t,S-1)}
&=\frac{1}{2}\spar{\nabla_tV_{\alpha}(t,S+2)+\nabla_tV_{\alpha}(t,S-2)-2\nabla_tV_{\alpha}(t,S)}\\
&=-\frac{C\sqrt{\alpha}}{4\sqrt{t}}\spar{\exp\rpar{\frac{(S+2)^2}{4\alpha t}}+\exp\rpar{\frac{(S-2)^2}{4\alpha t}}-2\exp\rpar{\frac{S^2}{4\alpha t}}}\\
&\leq 0.\tag{convexity}
\end{align*}

Combining the two cases, we can upper bound $f(t,S)$ by $\bar\nabla_S V_{\alpha}(t,S+1)-\bar\nabla_S V_{\alpha}(t,S-1)$, which completes the proof. 
\end{proof}

Next, we present the proof of Lemma~\ref{lemma:vanish}, which bounds the residual term $\Delta_t$ defined in (\ref{eq:residual}). 

\vanish*

\begin{proof}[Proof of Lemma~\ref{lemma:vanish}]
We restate the definition of $\Delta_t$ for easier reference. 
\begin{equation*}
\Delta_t=\bar\nabla_tV_\alpha(t,S_{t-1})+\frac{1}{2}\bar\nabla_{SS}V_\alpha(t,S_{t-1})+G^{-1}\lambda\spar{\bar\nabla_S V_{\alpha}(t,S_{t-1}+1)-\bar\nabla_S V_{\alpha}(t,S_{t-1}-1)}.
\end{equation*}
Let us define a function $g(t,S)$ as
\begin{equation*}
g(t,S)\defeq\frac{1}{2}V_{\alpha}(t,S+1)+\frac{1}{2}V_{\alpha}(t,S-1)-V_{\alpha}(t-1,S)+\frac{\lambda}{2G}\spar{V_{\alpha}(t,S+2)+V_{\alpha}(t,S-2)-2V_{\alpha}(t,S)},
\end{equation*}
then from the definition of discrete derivatives, $\Delta_t=g(t,S_{t-1})$. Also note that $g(t,S)$ is even in $S$, so we can only focus on positive values of $S$. The rest of the proof will show $g(t,S)\leq 0$ for all $t\in\N_+$ and $S\geq 0$. 

Let us start from the special case, $t=1$. $S$ can only take the value $0$, therefore $g(1,S)=g(1,0)$. We now present a general result that upper bounds $g(t,0)$ for all $t\in\N_+$:
\begin{align*}
g(t,0)&=V_{\alpha}(t,1)-V_{\alpha}(t-1,0)+G^{-1}\lambda V_{\alpha}(t,2)-G^{-1}\lambda V_{\alpha}(t,0)\\
&=C\sqrt{\alpha t}\spar{2\int_0^{1/\sqrt{4\alpha t}}\rpar{\int_0^u\exp(x^2)dx}du+\frac{2\lambda}{G}\int_0^{1/\sqrt{\alpha t}}\rpar{\int_0^u\exp(x^2)dx}du+\sqrt{\frac{t-1}{t}}-1}\\
&\leq C\sqrt{\alpha t}\spar{2\cdot\frac{1}{2}\cdot\frac{1}{\sqrt{4\alpha t}}\int_0^{1/\sqrt{4\alpha t}}\exp(x^2)dx+\frac{2\lambda}{G}\cdot\frac{1}{2}\cdot\frac{1}{\sqrt{\alpha t}}\int_0^{1/\sqrt{\alpha t}}\exp(x^2)dx+\sqrt{\frac{t-1}{t}}-1}\tag{erfi($x$) is increasing and convex on $\R_+$}\\
&\leq C\sqrt{\alpha t}\spar{\frac{1}{4\alpha t}\exp\rpar{\frac{1}{4\alpha t}}+\frac{\lambda}{G\alpha t}\exp\rpar{\frac{1}{\alpha t}}+\sqrt{\frac{t-1}{t}}-1}.
\end{align*}
Since $\sqrt{1+x}\leq 1+x/2$ for all $x\geq -1$, we have $\sqrt{(t-1)/t}-1\leq -1/(2t)$. Therefore, if $\alpha\geq 4\lambda G^{-1}+2$, then
\begin{equation}\label{eq:gzero}
g(t,0)\leq C\sqrt{\alpha t}\spar{\frac{\lambda G^{-1}+(1/4)}{\alpha t}\exp\rpar{\frac{1}{\alpha t}}-\frac{1}{2t}}\leq \frac{C\sqrt{\alpha}}{\sqrt{t}}\spar{\frac{\lambda G^{-1}+(1/4)}{\alpha }\exp\rpar{\frac{1}{2}}-\frac{1}{2}}\leq 0. 
\end{equation}
As its special case, we have $g(1,0)\leq 0$, which concludes the proof of the special case ($t=1$). 

Next, we prove $g(t,S)\leq 0$ for general $t$, i.e., $t\geq 2$. Our overall strategy is to show that for all $0\leq S\leq t-1$, $g(t,S)\leq g(t,0)$, and then using the argument above we have $g(t,0)\leq 0$. Concretely, let us calculate the first and second order derivatives of $g(t,S)$, using Lemma~\ref{lemma:derivs}.
\begin{multline*}
\nabla_S g(t,S)=\frac{C}{2}\spar{\int_0^{(S+1)/\sqrt{4\alpha t}}\exp(x^2)dx+\int_0^{(S-1)/\sqrt{4\alpha t}}\exp(x^2)dx-2\int_0^{S/\sqrt{4\alpha (t-1)}}\exp(x^2)dx}\\
+\frac{\lambda C}{2G}\spar{\int_0^{(S+2)/\sqrt{4\alpha t}}\exp(x^2)dx+\int_0^{(S-2)/\sqrt{4\alpha t}}\exp(x^2)dx-2\int_0^{S/\sqrt{4\alpha t}}\exp(x^2)dx},
\end{multline*}
\begin{align}
\nabla_{SS}g(t,S)\nonumber
&=\frac{C}{4\sqrt{\alpha t}}\Bigg[\frac{\lambda}{G}\exp\rpar{\frac{(S+2)^2}{4\alpha t}}+\exp\rpar{\frac{(S+1)^2}{4\alpha t}}-\frac{2\lambda}{G}\exp\rpar{\frac{S^2}{4\alpha t}}\nonumber\\
&\hspace{40pt}+\exp\rpar{\frac{(S-1)^2}{4\alpha t}}+\frac{\lambda}{G}\exp\rpar{\frac{(S-2)^2}{4\alpha t}}\Bigg]-\frac{C}{2\sqrt{\alpha(t-1)}}\exp\rpar{\frac{S^2}{4\alpha (t-1)}}\nonumber\\
&=\frac{C}{2\sqrt{\alpha t}}\exp\rpar{\frac{S^2}{4\alpha t}}\Bigg[\frac{\lambda}{G}\exp\rpar{\frac{1}{\alpha t}}\cosh\rpar{\frac{S}{\alpha t}}+\exp\rpar{\frac{1}{4 \alpha t}}\cosh\rpar{\frac{S}{2\alpha t}}\nonumber\\
&\hspace{40pt}-\frac{\lambda}{G}-\sqrt{\frac{t}{t-1}}\exp\rpar{\frac{S^2}{4\alpha t(t-1)}}\Bigg].\label{eq:defh}
\end{align}

Notice that $\nabla_Sg(t,0)=0$. To proceed, we aim to prove $\nabla_{SS}g(t,S)\leq 0$ for all $S\geq 0$, which then shows $g(t,S)\leq g(t,0)$. To this end, we will show the sum inside the bracket in (\ref{eq:defh}) is negative. Denote it as $h(t,S)$, and more specifically, 
\begin{equation*}
h(t,S)\defeq \frac{\lambda}{G}\exp\rpar{\frac{1}{\alpha t}}\cosh\rpar{\frac{S}{\alpha t}}+\exp\rpar{\frac{1}{4 \alpha t}}\cosh\rpar{\frac{S}{2\alpha t}}-\frac{\lambda}{G}-\sqrt{\frac{t}{t-1}}\exp\rpar{\frac{S^2}{4\alpha t(t-1)}}.
\end{equation*}
The rest of the proof is divided into two steps: we first prove ($i$) $h(t,0)\leq 0$; and then prove ($ii$) $\nabla_S h(t,S)\leq 0$ for all $S\geq 0$.

\paragraph{Step 1: prove $h(t,0)\leq 0$.} From the definition of $h(t,S)$,
\begin{equation*}
h(t,0)=\frac{\lambda}{G}\exp\rpar{\frac{1}{\alpha t}}+\exp\rpar{\frac{1}{4 \alpha t}}-\frac{\lambda}{G}-\sqrt{\frac{t}{t-1}}. 
\end{equation*}
Letting $x=1/t$, then to prove $h(t,0)\leq 0$ for all $t\geq 2$, it suffices to prove
\begin{equation*}
\psi(x)\defeq \frac{\lambda}{G}\exp\rpar{\frac{x}{\alpha}}+\exp\rpar{\frac{x}{4\alpha}}-\frac{\lambda}{G}-\sqrt{\frac{1}{1-x}}\leq 0,
\end{equation*}
on the range $x\in(0,1/2]$. $\psi(0)=0$, and
\begin{equation*}
\nabla_x\psi(x)=\frac{\lambda}{\alpha G}\exp\rpar{\frac{x}{\alpha}}+\frac{1}{4\alpha}\exp\rpar{\frac{x}{4\alpha}}-\frac{1}{2}(1-x)^{-3/2}\leq \frac{4\lambda G^{-1}+1}{4\alpha}\exp\rpar{\frac{1}{2\alpha}}-\frac{1}{2},
\end{equation*}
which is negative when $\alpha\geq 4\lambda G^{-1}+2$. Therefore, $h(t,0)\leq 0$ for all $t\geq 2$. 

\paragraph{Step 2: prove $\nabla_S h(t,S)\leq 0$.} Taking the derivative of $h(t,S)$, 
\begin{align*}
\nabla_S h(t,S)&=\frac{\lambda}{\alpha t G}\exp\rpar{\frac{1}{\alpha t}}\sinh\rpar{\frac{S}{\alpha t}}+\frac{1}{2\alpha t}\exp\rpar{\frac{1}{4 \alpha t}}\sinh\rpar{\frac{S}{2\alpha t}}-\sqrt{\frac{t}{t-1}}\cdot\frac{S}{2\alpha t(t-1)}\exp\rpar{\frac{S^2}{4\alpha t(t-1)}}\\
&\leq \rpar{\frac{\lambda}{\alpha t G}+\frac{1}{2\alpha t}}\exp\rpar{\frac{1}{\alpha t}}\sinh\rpar{\frac{S}{\alpha t}}-\frac{S}{2\alpha t^2}\sqrt{\frac{t}{t-1}}.
\end{align*}
Note that for all $x$, $\exp(-x)\geq 1-x$, therefore for all $0\leq x< 1$, $\exp(x/2)\leq \sqrt{1/(1-x)}$. Assigning $x$ to $1/t$ which is less than 1, we have for all $\alpha\geq 2$, 
\begin{equation*}
\exp\rpar{\frac{1}{\alpha t}}\leq \exp\rpar{\frac{1}{2 t}}\leq \sqrt{\frac{t}{t-1}}.
\end{equation*}
Moreover, for all $0\leq x\leq 1$, $\sinh(x)\leq 2x$. Therefore, 
\begin{equation*}
\nabla_S h(t,S)\leq \sqrt{\frac{t}{t-1}}\spar{\frac{\lambda G^{-1}+(1/2)}{\alpha t}\sinh\rpar{\frac{S}{\alpha t}}-\frac{S}{2\alpha t^2}}\leq \frac{S}{\alpha ^2t^2}\sqrt{\frac{t}{t-1}}\spar{2\lambda G^{-1}+1-\frac{\alpha}{2}}.
\end{equation*}
When $\alpha\geq 4\lambda G^{-1}+2$, $\nabla_S h(t,S)\leq 0$ for all $t\geq 2$ and $S\geq 0$. 

Concluding the above two steps, we have shown $h(t,S)\leq 0$. Plugging it back into (\ref{eq:defh}), we have $\nabla_{SS}g(t,S)\leq 0$, which shows that for all $t\geq 2$ and $S\geq 0$, $g(t,S)\leq g(t,0)$. Finally, $g(t,0)\leq 0$ following (\ref{eq:gzero}).
\end{proof}

Now, given the two important lemmas above (Lemma~\ref{lemma:switch} and \ref{lemma:vanish}), our Theorem~\ref{thm:main} follows from a standard loss-regret duality. Details are presented below.

\main*

\begin{proof}[Proof of Theorem~\ref{thm:main}]
Combining Lemma~\ref{lemma:ito}, \ref{lemma:switch} and \ref{lemma:vanish}, we have
\begin{equation*}
\sum_{t=1}^T\rpar{g_tx_t+\lambda\abs{x_t-x_{t+1}}}\leq -G \cdot V_\alpha(T,S_T).
\end{equation*}
Consider $V_\alpha(T,S_T)$ as a function of $S_T$; let us write $V^*_{\alpha,T}(\cdot)$ as its Fenchel conjugate. Then, the augmented regret can be bounded as
\begin{align*}
\reg^\lambda_T(u)&=\sum_{t=1}^Tg_t(x_t-u)+\lambda\sum_{t=1}^{T-1}\abs{x_t-x_{t+1}}\\
&\leq G\cdot uS_T+\sum_{t=1}^T\rpar{g_tx_t+\lambda\abs{x_t-x_{t+1}}}\\
&\leq G\spar{uS_T-V_\alpha(T,S_T)}\leq G\cdot V^*_{\alpha,T}(u).
\end{align*}

The last step is to bound $V^*_{\alpha,T}(u)$, which also follows from a standard proof strategy. 
\begin{equation*}
V^*_{\alpha,T}(u)=\sup_{S\in\R} uS-V_{\alpha}(T,S).
\end{equation*}
It is clear that the supremum is uniquely achieved; let $S^*$ be the maximizing argument. Then, 
\begin{equation*}
u=\nabla_S V_\alpha(T,S^*)=C\int_{0}^{S^*/\sqrt{4\alpha T}}\exp\rpar{x^2}dx.
\end{equation*}
If we define $\erfi(z)=\int_0^z\exp(x^2)dx$ (note that it a scaled version of the conventional \emph{imaginary error function}), then $S^*=\sqrt{4\alpha T}\cdot \erfi^{-1}\rpar{ uC^{-1}}$. 
\begin{equation*}
V^*_{\alpha,T}(u)=uS^*-V_\alpha(T,S^*)\leq uS^*-V_\alpha(T,0)=C\sqrt{\alpha T}+\abs{u}\sqrt{4\alpha T}\cdot \erfi^{-1}\rpar{uC^{-1}}.
\end{equation*}

Finally, as shown in \cite[Theorem~4]{zhang2022pde}, $\erfi^{-1}(x)\leq 1+\sqrt{\log(1+x)}$. Combining the above completes the proof. 
\end{proof}

\subsection{Conversion of loss-regret trade-offs}\label{subsection:conversion}

In this subsection we discuss the conversion of loss-regret trade-offs in unconstrained OLO. Our Theorem~\ref{thm:main} guarantees a loss bound $\reg^\lambda_T(0)=O(\sqrt{T})$ and an accompanying regret bound $\reg^\lambda_T(u)=O\rpar{\abs{u}\sqrt{T\log\abs{u}}}$. By a doubling trick (effectively, a meta-algorithm), we can turn such guarantees into $\reg^\lambda_T(0)=O(1)$ and $\reg^\lambda_T(u)=O\rpar{\abs{u}\sqrt{T\log(\abs{u}T)}}$. These can be directly compared to \cite{zhang2022adversarial}. Concretely, we present the classical doubling trick as Algorithm~\ref{alg:doubling}. 

\begin{algorithm*}[ht]
\caption{Conversion of loss-regret trade-offs. \label{alg:doubling}}
\begin{algorithmic}[1]
\REQUIRE A hyperparameter $C>0$, and a base unconstrained OLO algorithm $\A$. Here we define $\A$ as the algorithm considered in Theorem~\ref{thm:main}, with $\alpha=8\lambda G^{-1}+2$.
\FOR{$m=0,1,2,\ldots$}
\STATE Initialize a copy of $\A$ as $\A_m$, whose hyperparameter is set to $C\cdot 2^{-m}$. 
\STATE Run $\A_m$ for $2^m$ rounds: $t=2^m,2^m+1,\ldots,2^{m+1}-1$. 
\ENDFOR
\end{algorithmic}
\end{algorithm*}

\begin{theorem}\label{thm:doubling}
Let $\alpha=8\lambda G^{-1}+2$. With any hyperparameter $C>0$, Algorithm~\ref{alg:doubling} guarantees for all $u\in\R$ and $T\in\N_+$, 
\begin{equation*}
\mathrm{Regret}^\lambda_T(u)\leq \frac{\sqrt{2\alpha}G}{\sqrt{2}-1}\spar{C+\abs{u}\sqrt{T}\rpar{\sqrt{8\log\rpar{1+\frac{\abs{u}T}{C}}}+2\sqrt{2}}}.
\end{equation*}
\end{theorem}

\begin{proof}[Proof of Theorem~\ref{thm:doubling}]
Algorithm~\ref{alg:doubling} divides the time axis into intervals of doubling lengths. On the $m$-th interval, following Theorem~\ref{thm:main}, Algorithm~\ref{alg:doubling} guarantees
\begin{align*}
\sum_{t=2^m}^{2^{m+1}-1}\spar{g_t(x_t-u)+\lambda\abs{x_{t}-x_{t+1}}}
&\leq \sum_{t=2^m}^{2^{m+1}-1}g_t(x_t-u)+2\lambda\sum_{t=2^m}^{2^{m+1}-2}\abs{x_t-x_{t+1}}\tag{$x_{2^{m+1}}=x_{2^m}=0$; Triangle inequality}\\
&\leq \sqrt{\alpha}G\spar{\frac{C}{\sqrt{2^m}}+\abs{u}\sqrt{2^m}\rpar{\sqrt{4\log\rpar{1+\frac{\abs{u}\cdot 2^m}{C}}}+2}}.
\end{align*}

Now consider any time horizon $T$.
\begin{align*}
\reg^\lambda_T(u)&\leq\sum_{m=0}^{\ceil{\log_2 T}}\sum_{t=2^m}^{2^{m+1}-1}\spar{g_t(x_t-u)+\lambda\abs{x_t-x_{t+1}}}\\
&\leq \sqrt{\alpha}G\sum_{m=0}^{\ceil{\log_2 T}}\spar{\frac{C}{\sqrt{2^m}}+\abs{u}\sqrt{2^m}\rpar{\sqrt{4\log\rpar{1+\frac{\abs{u}\cdot 2^m}{C}}}+2}}\\
&\leq \sqrt{\alpha}G\spar{C\sum_{m=0}^{\ceil{\log_2 T}}\rpar{\frac{1}{\sqrt{2}}}^m+\abs{u}\rpar{\sqrt{4\log\rpar{1+\frac{\abs{u}T}{C}}}+2}\sum_{m=0}^{\ceil{\log_2 T}}\sqrt{2^m}}\\
&\leq \frac{\sqrt{2\alpha}G}{\sqrt{2}-1}\spar{C+\abs{u}\sqrt{T}\rpar{\sqrt{8\log\rpar{1+\frac{\abs{u}T}{C}}}+2\sqrt{2}}}.\qedhere
\end{align*}
\end{proof}

Let us compare it to Theorem~\ref{theorem:1d}, i.e., \cite[Theorem 1]{zhang2022adversarial}, which guarantees
\begin{equation*}
\reg^\lambda_T(u)\leq (G+\lambda)\spar{C+\abs{u}\sqrt{2T}\rpar{\frac{3}{2}+\log\frac{\sqrt{2}\abs{u}T^{5/2}}{C}}}.
\end{equation*}
If we only care about the dependence on $\abs{u}$ and $T$, then with the same loss bound $\reg^\lambda_T(0)=O(1)$, our algorithm improves the regret bound $\reg^\lambda_T(u)$ from $O(\abs{u}\sqrt{T}\log(\abs{u}T))$ to $O(\abs{u}\sqrt{T \log(\abs{u}T)})$. The latter matches a lower bound \cite{orabona2013dimension}, therefore achieves Pareto-optimality in this regime.

\subsection[]{Details on the continuous-time derivation}\label{subsection:detail_continuous}

In Step 3 of Section~\ref{subsection:continuous}, we need to perform second-order Taylor approximations on the scaled recursion (\ref{eq:scaled}). Details are included here for completeness.  
\begin{equation*}
V(t-\eps^2,S)=V(t,S)-\eps^2\nabla_tV(t,S)+o(\eps^2),
\end{equation*}
\begin{equation*}
V(t,S-\eps g)=V(t,S)-\eps g \nabla_SV(t,S)+\frac{1}{2}\eps^2g^2\nabla_{SS}V(t,S)+o(\eps^2),
\end{equation*}
\begin{equation*}
\bar \nabla^\eps_{S} V(t,S)=\frac{1}{2\eps}\spar{V(t,S+\eps)-V(t,S-\eps)}=\nabla_SV(t,S)+o(\eps),
\end{equation*}
\begin{equation*}
\bar \nabla^\eps_S V(t+\eps^2,S-\eps g)=\frac{1}{2\eps}\spar{V(t+\eps^2,S-\eps g+\eps)-V(t+\eps^2,S-\eps g-\eps)},
\end{equation*}
where
\begin{equation*}
V(t+\eps^2,S-\eps g+\eps)=V(t,S)+\eps^2\nabla_tV(t,S)+(1-g)\eps\nabla_S V(t,S)+\frac{1}{2}(1-g)^2\eps^2\nabla_{SS}V(t,S)+o(\eps^2),
\end{equation*}
\begin{equation*}
V(t+\eps^2,S-\eps g-\eps)=V(t,S)+\eps^2\nabla_tV(t,S)+(-1-g)\eps\nabla_S V(t,S)+\frac{1}{2}(1+g)^2\eps^2\nabla_{SS}V(t,S)+o(\eps^2).
\end{equation*}

\section{Extension to general OLO settings}\label{section:extension}

This section presents extensions of our 1D unconstrained OLO algorithm to more general settings.

\subsection{Constrained domain}\label{subsection:extension}

First, consider a constrained domain $\calX\subset\R$. We can use a well-known black-box reduction \cite{cutkosky2018black,cutkosky2020parameter} on top of our 1D unconstrained algorithm (Algorithm~\ref{alg:1d}), such that the \emph{exact} bound in Theorem~\ref{thm:main} carries over (w.r.t. any constrained comparator $u\in\calX$). Concretely, the pseudo-code is shown as Algorithm~\ref{alg:constraint}, where $\Pi$ denotes the absolute value projection in 1D.

\begin{algorithm*}[ht]
\caption{1D constrained OLO with switching costs. \label{alg:constraint}}
\begin{algorithmic}[1]
\REQUIRE A hyperparameter $C>0$, a closed and convex domain $\calX\subset\R$, and an unconstrained algorithm $\A$ (Algorithm~\ref{alg:1d} induced by $V_{4\lambda G^{-1}+2}$ and the hyperparameter $C$). Let $x^*$ be an arbitrary point in $\calX$.
\FOR{$t=1,2,\ldots$}
\STATE Query $\A$ for its prediction $\tilde x_t$. 
\STATE Predict $x_t=\Pi_\calX(\tilde x_t+x^*)$ and receive a loss gradient $g_t$. 
\STATE Define a surrogate loss gradient $\tilde g_t$ as
\begin{equation*}
\tilde g_t=\begin{cases}
g_t, &\textrm{if~}g_t(\tilde x_t+x^*)\geq g_tx_t,\\
0, &\textrm{otherwise},
\end{cases}
\end{equation*}
and send $\tilde g_t$ to $\A$. 
\ENDFOR
\end{algorithmic}
\end{algorithm*}

Similar strategies apply to higher-dimensional problems, but here we emphasize the 1D special case due to an additional feature: if the domain $\calX$ has a finite diameter $D$, then the switching cost alone of the combined algorithm has a $\tilde O(D\sqrt{\tau})$ bound on any time interval of length $\tau$. This could be useful when switching costs have high priority \cite{sherman2021lazy,wang2021online} and should be independently bounded. Moreover, it allows the combination of comparator adaptive algorithms \cite{zhang2022adversarial} in settings with long term prediction effects (e.g., switching cost or memory). 

\begin{theorem}\label{thm:constraint}
Let $x^*$ be an arbitrary point in $\calX$. For all $C>0$, Algorithm~\ref{alg:constraint} guarantees
\begin{equation*}
\mathrm{Regret}^\lambda_T(u)\leq \sqrt{(4\lambda G+2G^2) T}\spar{C+\abs{u-x^*}\rpar{\sqrt{4\log\rpar{1+\frac{\abs{u-x^*}}{C}}}+2}},
\end{equation*}
for all $u\in\calX$ and $T\in \N_+$. Moreover, if $\calX$ has a finite diameter $D$, then on any time interval $[T_1:T_2]\subset\N_+$, the same algorithm guarantees 
\begin{equation*}
\sum_{t=T_1}^{T_2-1}\abs{x_t-x_{t+1}}\leq 22\sqrt{T_2-T_1}\spar{2D+C+2D\sqrt{\log(1+DC^{-1})}}.
\end{equation*}
\end{theorem}

Before presenting the proof, let us discuss its technical significance. Typically, the constrained-to-unconstrained reduction is used as a black box. However, the second part of Theorem~\ref{thm:constraint} relies on a \emph{non-black-box} use of this approach -- we characterize how this reduction implicitly controls the unconstrained base algorithm, resulting in the ``concentration'' of its sufficient statistic $S_t$ (i.e., $S_t=O(\sqrt{t})$), as if losses are stochastic. Such an observation could be of independent interest. 

\begin{proof}[Proof of Theorem~\ref{thm:constraint}]
The first part of the theorem directly follows from \cite[Theorem~2]{cutkosky2020parameter} and the contraction property of 1D projections. As for the second part, we divide the proof into five steps. 

\paragraph{Step 1: the ``concentration'' of $S_t$} Without loss of generality, assume $S_{t-1}\geq 0$. Considering the prediction $\tilde x_t=\bar\nabla_S V_\alpha(t,S_{t-1})$ of the unconstrained base algorithm, there are two cases.

\begin{itemize}
\item \textbf{Case 1: $\tilde x_t\leq D$.}\quad Due to convexity,
\begin{equation*}
\tilde x_t=\bar\nabla_S V_{\alpha}(t,S_{t-1})= C\sqrt{\alpha t}\int_{(S_{t-1}-1)/\sqrt{4\alpha t}}^{(S_{t-1}+1)/\sqrt{4\alpha t}}\rpar{\int_0^u\exp(x^2)dx}du\geq C\int_0^{S_{t-1}/\sqrt{4\alpha t}}\exp(x^2)dx. 
\end{equation*}
Similar to the proof of Theorem~\ref{thm:main}, if we define $\erfi(z)=\int_0^z\exp(x^2)dx$, then $S_{t-1}\leq \sqrt{4\alpha t}\cdot \erfi^{-1}(DC^{-1})$. As for the next round, $\abs{S_{t}}\leq S_{t-1}+\abs{g_t}/G\leq \sqrt{4\alpha t}\cdot \erfi^{-1}(DC^{-1})+1$.
\item \textbf{Case 2: $\tilde x_t> D$.}\quad In this case, since $x^*\in\calX$, we have $\tilde x_t+x^*$ larger than the maximum element of $\calX$, leading to $\tilde x_t+x^*>x_t$. Due to the definition of the surrogate loss, $\tilde g_t\geq 0$. Therefore, $\abs{S_t}\leq \max\{\abs{S_{t-1}},\abs{\tilde g_t/G}\}\leq \max\{\abs{S_{t-1}},1\}$.
\end{itemize}

Combining the two cases and their analogous arguments for $S_{t-1}\leq 0$, we can see that for all $t$, $\abs{S_t}\leq \max\left\{\sqrt{4\alpha t}\cdot \erfi^{-1}(DC^{-1})+1, \abs{S_{t-1}}, 1\right\}$. By induction, we obtain for all $t$,
\begin{equation*}
\abs{S_{t}}\leq \sqrt{4\alpha t}\cdot \erfi^{-1}(DC^{-1})+1.
\end{equation*}

\paragraph{Step 2: bounding the switching cost using $S_t$} Still, assume $S_{t-1}\geq 0$ without loss of generality. From Lemma~\ref{lemma:switch}, 
\begin{align*}
\abs{x_t-x_{t+1}}
&\leq \bar\nabla_S V_{\alpha}(t,S_{t-1}+1)-\bar\nabla_S V_{\alpha}(t,S_{t-1}-1)\\
&=C\sqrt{\alpha t}\spar{\int_{S_{t-1}/\sqrt{4\alpha t}}^{(S_{t-1}+2)/\sqrt{4\alpha t}}\rpar{\int_0^u\exp(x^2)dx}du-\int_{(S_{t-1}-2)/\sqrt{4\alpha t}}^{S_{t-1}/\sqrt{4\alpha t}}\rpar{\int_0^u\exp(x^2)dx}du}\\
&\leq C\sqrt{\alpha t}\spar{\frac{2}{\sqrt{4\alpha t}}\int_0^{(S_{t-1}+2)/\sqrt{4\alpha t}}\exp(x^2)dx-\frac{2}{\sqrt{4\alpha t}}\int_0^{(S_{t-1}-2)/\sqrt{4\alpha t}}\exp(x^2)dx}\\
&=C\int_{(S_{t-1}-2)/\sqrt{4\alpha t}}^{(S_{t-1}+2)/\sqrt{4\alpha t}}\exp(x^2)dx\leq \frac{2C}{\sqrt{\alpha t}}\exp\rpar{\frac{(S_{t-1}+2)^2}{4\alpha t}}.
\end{align*}

\paragraph{Step 3: plug in the concentration of $S_t$} Next, we use the upper bound on $S_{t-1}$ to show that $\abs{x_t-x_{t+1}}=O(C t^{-1/2}\exp[(\erfi^{-1}(DC^{-1}))^2])$. To this end, we discuss two cases regarding how the ``concentration'' bound (i.e., $O(\sqrt{t})$) compares to the trivial bound (i.e., $S_t\leq t$). 

\begin{itemize}
\item \textbf{Case 1: $\sqrt{4\alpha t}\cdot \erfi^{-1}(DC^{-1})\geq t$.}\quad In this case, note that $S_{t-1}+1\leq t$ and $\alpha\geq 2$, 
\begin{equation*}
\abs{x_t-x_{t+1}}\leq \frac{2C}{\sqrt{\alpha t}}\exp\rpar{\frac{(S_{t-1}+2)^2}{4\alpha t}}=\frac{2C}{\sqrt{\alpha t}}\exp\rpar{\frac{S^2_{t-1}}{4\alpha t}}\exp\rpar{\frac{4S_{t-1}+4}{4\alpha t}}\leq \frac{2\sqrt{e}C}{\sqrt{\alpha t}}\exp\rpar{\frac{S^2_{t-1}}{4\alpha t}}.
\end{equation*}

Since $\sqrt{4\alpha t}\cdot \erfi^{-1}(DC^{-1})\geq t$, we have $t\leq 4\alpha (\erfi^{-1}(DC^{-1}))^2$. Therefore, 
\begin{equation*}
\exp\rpar{\frac{S^2_{t-1}}{4\alpha t}}\leq \exp\rpar{\frac{t}{4\alpha}}\leq \exp\spar{\rpar{\erfi^{-1}(DC^{-1})}^2},
\end{equation*}
\begin{equation*}
\abs{x_t-x_{t+1}}\leq \frac{2\sqrt{e}C}{\sqrt{\alpha t}}\exp\spar{\rpar{\erfi^{-1}(DC^{-1})}^2}.
\end{equation*}

\item \textbf{Case 2: $\sqrt{4\alpha t}\cdot \erfi^{-1}(DC^{-1})< t$.}\quad Plugging the $O(\sqrt{t})$ bound for $S_{t-1}$ into $\abs{x_t-x_{t+1}}$, 
\begin{align*}
\abs{x_t-x_{t+1}}&\leq \frac{2C}{\sqrt{\alpha t}}\exp\rpar{\frac{(\sqrt{4\alpha (t-1)}\cdot \erfi^{-1}(DC^{-1})+3)^2}{4\alpha t}}\\
&\leq \frac{2C}{\sqrt{\alpha t}}\exp\spar{\rpar{\erfi^{-1}(DC^{-1})}^2}\exp\rpar{\frac{6t+9}{4\alpha t}}\\
&\leq \frac{2e^2C}{\sqrt{\alpha t}}\exp\spar{\rpar{\erfi^{-1}(DC^{-1})}^2}.
\end{align*}
\end{itemize}

Combining the above, we have
\begin{equation*}
\abs{x_t-x_{t+1}}\leq \frac{2e^2C}{\sqrt{\alpha t}}\exp\spar{\rpar{\erfi^{-1}(DC^{-1})}^2}.
\end{equation*}

\paragraph{Step 4: upper-bound $\exp[(\erfi^{-1}(x))^2]$.} Let us consider a related function $\int_0^x\erfi(u)du$. Using integration by parts, 
\begin{align*}
\int_0^x\erfi(u)du&=u\cdot\erfi(u)\bigg|_{u=0}^x-\int_0^xu\exp(u^2)du\\
&=x\cdot \erfi(x)-\frac{1}{2}\exp(x^2)+\frac{1}{2}.
\end{align*}
Plugging in $x=\erfi^{-1}(DC^{-1})$, we have
\begin{align*}
\exp\spar{\rpar{\erfi^{-1}(DC^{-1})}^2}&= 2DC^{-1}\cdot\erfi^{-1}(DC^{-1})+1-2\int_0^{\erfi^{-1}(DC^{-1})}\erfi(u)du\\
&\leq 2DC^{-1}\cdot\erfi^{-1}(DC^{-1})+1. 
\end{align*}
Then, as we did in Theorem~\ref{thm:main}, we plug in $\erfi^{-1}(x)\leq 1+\sqrt{\log(1+x)}$ and obtain
\begin{equation*}
\abs{x_t-x_{t+1}}\leq \frac{11}{\sqrt{t}}\left\{2D\spar{1+\sqrt{\log(1+DC^{-1})}}+C\right\}.
\end{equation*}

\paragraph{Step 5: final bits.} Note that
\begin{equation*}
\sum_{t=T_1}^{T_2-1}\frac{1}{\sqrt{t}}\leq \int_{T_1-1}^{T_2-1}\frac{1}{\sqrt{x}}dx\leq 2\sqrt{T_2-1}-2\sqrt{T_1-1}\leq 2\sqrt{T_2-T_1}.
\end{equation*}
Combining it with our bound on $\abs{x_t-x_{t+1}}$ completes the proof.
\end{proof}

\subsection{Higher dimensional domain}\label{subsection:higherd}

Now we present generalizations of our 1D algorithm to higher dimensional domains. We will primarily consider switching costs measured by the $L_1$ norm. This serves as a nice bridge towards our LEA approach and financial applications. Extensions to other norms, e.g., $L_2$ norm, is sketched at the end. 

Concretely, let the domain $\calX=\R^d$, $\norms{g_t}_\infty\leq G$, and the switching costs are measured by the $L_1$ norm. We run Algorithm~\ref{alg:1d} on each coordinate separately \cite{streeter2012no}, and scale the hyperparameter $C$ by $1/d$. The pseudo-code is presented as Algorithm~\ref{alg:higherd}.

\begin{algorithm*}[ht]
\caption{$d$-dimensional OLO with $L_1$ norm switching costs. \label{alg:higherd}}
\begin{algorithmic}[1]
\REQUIRE A hyperparameter $C>0$ and Algorithm~\ref{alg:1d}.
\STATE For each dimension $i\in[1:d]$, initialize a copy of Algorithm~\ref{alg:1d} as $\A_i$. It uses the hyperparameter $C/d$ and our potential $V_\alpha$, with $\alpha=4\lambda G^{-1}+2$. 
\FOR{$t=1,2,\ldots$}
\STATE For all $i$, query $\A_i$ and assign its prediction to $x_{t,i}$. Define a vector as $x_t=[x_{t,1},\ldots,x_{t,d}]\in\R^d$. 
\STATE Predict $x_t$ and receive a loss gradient $g_t=[g_{t,1},\ldots,g_{t,d}]$. 
\STATE For all $i$, send $g_{t,i}$ to $\A_i$ as a new surrogate loss gradient.
\ENDFOR
\end{algorithmic}
\end{algorithm*}

\begin{theorem}\label{thm:higherd}
For all $C>0$, Algorithm~\ref{alg:higherd} guarantees ($\alpha=4\lambda G^{-1}+2$)
\begin{equation*}
\sum_{t=1}^T\inner{g_t}{x_t-u}+\lambda\sum_{t=1}^{T-1}\norm{x_t-x_{t+1}}_1\leq G\sqrt{\alpha T}\spar{C+\norm{u}_1\rpar{\sqrt{4\log\rpar{1+\frac{\norm{u}_\infty d}{C}}}+2}},
\end{equation*}
for all $u\in\R^d$ and $T\in\N_+$.
\end{theorem}

\begin{proof}[Proof of Theorem~\ref{thm:higherd}]
We simply combine the regret on each coordinate: 
\begin{align*}
\sum_{t=1}^T\inner{g_t}{x_t-u}+\lambda\sum_{t=1}^{T-1}\norm{x_t-x_{t+1}}_1
&=\sum_{i=1}^d\spar{\sum_{t=1}^Tg_{t,i}(x_{t,i}-u_i)+\lambda\sum_{t=1}^{T-1}\abs{x_{t,i}-x_{t+1,i}}}\\
&\leq \sqrt{(4\lambda G+2G^2) T}\sum_{i=1}^d\spar{\frac{C}{d}+\abs{u_i}\rpar{\sqrt{4\log\rpar{1+\frac{\abs{u_i}d}{C}}}+2}}\\
&\leq \sqrt{(4\lambda G+2G^2) T}\spar{C+\norm{u}_1\rpar{\sqrt{4\log\rpar{1+\frac{\norm{u}_\infty d}{C}}}+2}}.\qedhere
\end{align*}
\end{proof}

As for $L_2$ norm switching costs, we can follow the polar decomposition technique from \cite[Section~2.2]{zhang2022adversarial}, which uses our 1D unconstrained OLO algorithm as the base algorithm. The only required modification is that the base algorithm should account for an extra regularization term. Concretely, instead of bounding the augmented regret (\ref{eq:def}), we should bound
\begin{equation*}
\sum_{t=1}^Tg_t(x_t-u)+\lambda\sum_{t=1}^{T-1}\abs{x_t-x_{t+1}}+\frac{\gamma}{\sqrt{t}}\sum_{t=1}^T\abs{x_t},
\end{equation*}
for any given weight $\gamma$. 

To this end, we can consider the Online Convex Optimization problem with switching costs, where the loss functions are defined by $l_t(x)=g_tx+\gamma t^{-1/2}\abs{x}$. Such a loss function is Lipschitz, therefore we can use the OCO-OLO reduction, and run our Algorithm~\ref{alg:1d} on its linearized surrogate. Details are omitted. 

\section{Details on LEA}\label{section:detail_lea}

In this section we present techniques that extend our 1D OLO algorithm to LEA with switching costs. We show that with a streamlined analysis, the general Banach version of the constrained domain reduction \cite{cutkosky2018black} can already convert 1D OLO algorithms to LEA, thus appears to be more general than the specialized techniques \cite{luo2015achieving,orabona2016coin}. Our approach is presented as Algorithm~\ref{alg:lea}. 

\begin{algorithm*}[ht]
\caption{Converting OLO to LEA via the constrained domain reduction. \label{alg:lea}}
\begin{algorithmic}[1]
\REQUIRE A prior $\pi=[\pi_1,\ldots,\pi_d]$ in the relative interior of $\Delta(d)$, and Algorithm~\ref{alg:constraint}. 
\STATE For each dimension $i\in[1:d]$, initialize a copy of Algorithm~\ref{alg:constraint} as $\A_i$. We set $\tilde \lambda=4\lambda$ and $\tilde G=2G$ as the switching cost weight and the Lipschitz constant considered by $A_i$. Moreover, $A_i$ uses the domain $\calX=\R_+$, the offset $x^*=\pi_i$, the hyperparameter $\pi_i$, and our potential $V_\alpha$, where $\alpha=4\tilde \lambda\tilde G^{-1}+2$. 
\FOR{$t=1,2,\ldots$}
\STATE For all $i$, query $\A_i$ and assign its prediction to $w_{t,i}$. Define the weight vector as $w_t=[w_{t,1},\ldots,w_{t,d}]\in\R_+^d$.

\STATE Compute the LEA prediction $x_t=[x_{t,1},\ldots,x_{t,d}]$ from
\begin{equation*}
x_{t,i}=\frac{w_{t,i}+\frac{1}{d}\max\{0,1-\norm{w_t}_1\}}{\max\{\norm{w_t}_1,1\}}.
\end{equation*}
\STATE Predict $x_t$ and receive a loss vector $g_t\in[-G,G]^d$. 
\STATE For all $i$, compute\label{line:surrogate}
\begin{equation*}
z_{t,i}=\begin{cases}
g_{t,i}-\norm{g_t}_\infty,&\textrm{if}~ \norm{w_t}_1<1,\\
g_{t,i},&\textrm{if}~ \norm{w_t}_1=1,\\
g_{t,i}+\norm{g_t}_\infty,&\textrm{if}~ \norm{w_t}_1>1,
\end{cases}
\end{equation*}
and return $z_{t,i}$ to $\A_i$ as a new surrogate loss. 
\ENDFOR
\end{algorithmic}
\end{algorithm*}

\subsection{An auxiliary lemma}

Before presenting the performance guarantee of Algorithm~\ref{alg:lea}, we first prove an auxiliary lemma. 

\begin{lemma}\label{lemma:f_div}
For all $x\geq 0$, 
\begin{equation*}
\abs{x-1}\log(1+\abs{x-1})\leq 2(1-x+x\log x). 
\end{equation*}
\end{lemma}

\begin{proof}[Proof of Lemma~\ref{lemma:f_div}]
Define $\lhs-\rhs$ as $h(x)$. Clearly, $h(1)=0$. When $x>1$, 
\begin{equation*}
h'(x)=1-\log x-x^{-1}.
\end{equation*}
It equals $0$ when $x=1$, and $h''(x)=(1-x)/x^2$ which is negative for all $x> 1$. Therefore, $h(x)\leq 0$ for all $x\geq 1$. 

As for the case of $x<1$, 
\begin{equation*}
h'(x)=-\log(2-x)-\frac{1-x}{2-x}-2\log x,
\end{equation*}
\begin{equation*}
h''(x)=-\frac{x^2-x+2}{(x-2)^2x}<0, 
\end{equation*}
therefore $h(x)\leq 0$ for all $0\leq x\leq 1$. 
\end{proof}

\subsection{Analysis of Algorithm~\ref{alg:lea}}\label{subsection:lea_analysis}

Next, we present our result on LEA with switching cost. 

\lea*

\begin{proof}[Proof of Theorem~\ref{thm:lea}]
We divide the proof into three steps. 

\paragraph{Step 1} The first step is to show that ($i$) for all $u\in\Delta(d)$, $\inner{g_t}{x_t-u}\leq \inner{z_t}{w_t-u}$; and ($ii$) $\norm{x_t-x_{t+1}}_1\leq O(\norm{w_t-w_{t+1}}_1)$. In this way, we can translate the LEA problem to a $d$-dimensional OLO problem with the loss vector $z_t$, despite not achieving the root KL bound yet. 

To prove the goal ($i$), we consider two cases, $\norm{w_t}_1<1$ and $\norm{w_t}_1>1$ (the case of $\norm{w_t}_1=1$ trivially holds). If $\norm{w_t}_1<1$, we have $x_t=w_t+d^{-1}(1-\norm{w_t}_1)$ and $z_t=g_t-\norm{g_t}_\infty$. 
\begin{equation*}
\inner{g_t}{x_t-u}=\inner{g_t}{w_t-u}+(1-\norm{w_t}_1)\rpar{\sum_{i}g_{t,i}/d},
\end{equation*}
\begin{equation*}
\inner{z_t}{w_t-u}=\inner{g_t}{w_t-u}+(1-\norm{w_t}_1)\norm{g_t}_\infty,
\end{equation*}
therefore $\inner{g_t}{x_t-u}\leq \inner{z_t}{w_t-u}$. As for the case of $\norm{w_t}_1>1$, similarly, $x_t=w_t/\norm{w_t}_1$, $z_t=g_t+\norm{g_t}_\infty$, $\inner{g_t}{x_t-u}=\inner{g_t}{w_t/\norm{w_t}_1-u}$, and $\inner{z_t}{w_t-u}=\inner{g_t}{w_t-u}+\norm{g_t}_\infty(\norm{w_t}_1-1)$. 
\begin{equation*}
\inner{g_t}{x_t-u}-\inner{z_t}{w_t-u}=-\rpar{\inner{g_t}{x_t}+\norm{g_t}_\infty}(\norm{w_t}_1-1)\leq 0.
\end{equation*}

Now consider the goal ($ii$). To avoid cluttered notations, define $a_t=w_t+d^{-1}\max\{0,1-\norm{w_t}_1\}$ and $A_t=\max\{\norm{w_t}_1,1\}$. Note that $A_t=\norm{a_t}_1$.
\begin{align*}
\norm{x_t-x_{t+1}}_1&=\norm{\frac{a_t}{A_t}-\frac{a_{t+1}}{A_{t+1}}}_1\\
&=\norm{\frac{(a_t-a_{t+1})A_{t+1}+a_{t+1}(A_{t+1}-A_t)}{A_tA_{t+1}}}_1\\
&\leq \frac{1}{A_t}\norm{a_t-a_{t+1}}_1+\frac{1}{A_t}(A_{t+1}-A_t)\leq 2\norm{a_t-a_{t+1}}_1.
\end{align*}
\begin{align*}
\norm{a_t-a_{t+1}}_1&=\norm{w_t+d^{-1}\max\{0,1-\norm{w_t}_1\}-w_{t+1}-d^{-1}\max\{0,1-\norm{w_{t+1}}_1\}}_1\\
&\leq \norm{w_t-w_{t+1}}_1+\abs{\max\{0,1-\norm{w_t}_1\}-\max\{0,1-\norm{w_{t+1}}_1\}}\\
&\leq \norm{w_t-w_{t+1}}_1+\abs{\norm{w_t}_1-\norm{w_{t+1}}_1}\leq 2\norm{w_t-w_{t+1}}_1.
\end{align*}
Therefore, $\norm{x_t-x_{t+1}}_1\leq 4\norm{w_t-w_{t+1}}_1$.

\paragraph{Step 2} The second step is to add up the regret bound for each coordinates. Consider the $i$-th coordinate. Note that $\abs{z_{t,i}}\leq 2G$. Using Theorem~\ref{thm:constraint}, for all $u_{1d}\in\R_+$, 
\begin{align*}
&\sum_{t=1}^Tz_{t,i}(w_{t,i}-u_{1d})+\tilde \lambda\sum_{t=1}^{T-1}\abs{w_{t,i}-w_{t+1,i}}\\
\leq~& \sqrt{(4\tilde \lambda \tilde G+2\tilde G^2) T}\spar{\pi_i+\abs{u_{1d}-\pi_i}\rpar{\sqrt{4\log\rpar{1+\frac{\abs{u_{1d}-\pi_i}}{\pi_i}}}+2}}\\
=~&\sqrt{(32 \lambda G+8 G^2) T}\spar{\pi_i+\abs{u_{1d}-\pi_i}\rpar{\sqrt{4\log\rpar{1+\frac{\abs{u_{1d}-\pi_i}}{\pi_i}}}+2}}.
\end{align*}
Then, by summing up all the coordinates, for all $u\in\Delta(d)$, 
\begin{align*}
&\sum_{t=1}^T\inner{g_t}{x_t-u}+\lambda\sum_{t=1}^{T-1}\norm{x_t-x_{t+1}}_1\\
\leq~&\sum_{t=1}^T\inner{z_t}{w_t-u}+4\lambda\sum_{t=1}^{T-1}\norm{w_{t}-w_{t+1}}_1\\
=~&\sum_{i=1}^d\spar{\sum_{t=1}^Tz_{t,i}(w_{t,i}-u_{i})+\tilde \lambda\sum_{t=1}^{T-1}\abs{w_{t,i}-w_{t+1,i}}}\\
\leq~& \sqrt{(32 \lambda G+8 G^2) T}\spar{1+2\norm{u-\pi}_1+2\sum_{i=1}^d\abs{u_{i}-\pi_i}\sqrt{\log\rpar{1+\frac{\abs{u_{i}-\pi_i}}{\pi_i}}}}\\
\leq~& \sqrt{(32 \lambda G+8 G^2) T}\spar{1+2\norm{u-\pi}_1+2\sqrt{\norm{u-\pi}_1}\sqrt{\sum_{i=1}^d\abs{u_{i}-\pi_i}\log\rpar{1+\frac{\abs{u_{i}-\pi_i}}{\pi_i}}}}.\tag{Cauchy-Schwarz}
\end{align*}
Observe that since $u$ and $\pi$ both belong to $\Delta(d)$, $\norm{u-\pi}_1\leq 2$. If we define a function $f$ as
\begin{equation*}
f\defeq\abs{x-1}\log(1+\abs{x-1}),
\end{equation*}
then using the standard definition of \emph{f-divergence}
\begin{equation*}
D_f(u||\pi)\defeq \sum_{i=1}^d\pi_i f\rpar{\frac{u_i}{\pi_i}},
\end{equation*}
we have
\begin{equation*}
\sum_{t=1}^T\inner{g_t}{x_t-u}+\lambda\sum_{t=1}^{T-1}\norm{x_t-x_{t+1}}_1= \spar{\sqrt{\tv(u||\pi)\cdot D_f(u||\pi)}+1}\cdot O\rpar{\sqrt{(\lambda G+G^2)T}}.
\end{equation*}

\paragraph{Step 3} The last step is to upper bound $D_f(u||\pi)$ by $\kl(u||\pi)$. To this end, notice that $\kl(u||\pi)=D_g(u||\pi)$, where
\begin{equation*}
g(x)\defeq 1-x+x\log x. 
\end{equation*}
By Lemma~\ref{lemma:f_div}, $f(x)\leq 2g(x)$ for all $x\geq 0$, therefore $D_f(u||\pi)\leq 2D_g(u||\pi)=2\kl(u||\pi)$. 
\end{proof}

\subsection{Discussion on Algorithm~\ref{alg:lea}}\label{subsection:lea_discussion}

Here are some discussions to conclude our LEA analysis. First, the surrogate loss $z_t$ defined in Line~\ref{line:surrogate} follows exactly the definition in \cite[Algorithm~3]{cutkosky2018black}. We adopt this choice just to show the power of this general reduction technique. However, one could use other choices of $z_t$ and obtain the same guarantee, although the empirical performance could be different. For example, one can use
\begin{equation*}
z_{t,i}=\begin{cases}
g_{t,i}-\max_i g_{t,i},&\textrm{if}~ \norm{w_t}_1<1,\\
g_{t,i},&\textrm{if}~ \norm{w_t}_1=1,\\
g_{t,i}-\min_i g_{t,i},&\textrm{if}~ \norm{w_t}_1>1,
\end{cases}
\end{equation*}
and clearly, the exact same proof still holds. Another possible choice is
\begin{equation*}
z_{t,i}=\begin{cases}
g_{t,i}-\sum_i g_{t,i},&\textrm{if}~ \norm{w_t}_1<1,\\
g_{t,i},&\textrm{if}~ \norm{w_t}_1=1,\\
g_{t,i}-\inner{g_t}{x_t},&\textrm{if}~ \norm{w_t}_1>1.
\end{cases}
\end{equation*}
This is more analogous to the surrogate losses in existing specialized approaches \cite{luo2015achieving,orabona2016coin}.

Also, to justify the improvement of $\sqrt{\tv\cdot\kl}$ over $\sqrt{\kl}$, here is an example. For all $d\geq 3$, define $p,q\in\Delta(d)$ from
\begin{equation*}
p_1=\frac{1}{\sqrt{\log d}},\quad 
q_1=\frac{1}{d\sqrt{\log d}},
\end{equation*}
and
\begin{equation*}
p_i=\frac{1-p_1}{d-1},\quad q_i=\frac{1-q_1}{d-1}, \quad \forall i\in[2:d].
\end{equation*}
Then, 
\begin{equation*}
\tv(p||q)=\frac{1}{2}\spar{\abs{p_1-q_1}+(d-1)\abs{\frac{1-p_1}{d-1}-\frac{1-q_1}{d-1}}}=\abs{p_1-q_1}=\frac{d-1}{d\sqrt{\log d}},
\end{equation*}
\begin{align*}
\kl(p||q)&=p_1\log\frac{p_1}{q_1}+(d-1)\cdot \frac{1-p_1}{d-1}\log\frac{1-p_1}{1-q_1}\\
&=\sqrt{\log d}+\rpar{1-\frac{1}{\sqrt{\log d}}}\log\rpar{1-\frac{d-1}{d\sqrt{\log d}-1}}\\
&\geq \sqrt{\log d}+\log\rpar{1-\frac{d}{d\sqrt{\log d}-1}}=\sqrt{\log d}-o(1).
\end{align*}

Since we also have
\begin{equation*}
\kl(p||q)=\sqrt{\log d}+(1-p_1)\log\frac{1-p_1}{1-q_1}\leq \sqrt{d},
\end{equation*}
we can combine the above and obtain $\tv(p||q)\cdot\kl(p||q)\leq 1$ and $\kl(p||q)\geq \sqrt{\log d}-o(1)$. If our comparator $u$ and prior $\pi$ take the value of $p$ and $q$ respectively, then even without switching costs, Theorem~\ref{thm:lea} saves a $(\log d)^{1/4}$ factor from the existing comparator adaptive bounds. 

\section{Application to portfolio selection}\label{section:financial}

To complement our theoretical results, we present applications to a portfolio selection problem with transaction costs.\footnote{Code is available at \url{https://github.com/zhiyuzz/NeurIPS2022-Adaptive-Switching}.} Online portfolio selection has been studied by multiple communities, resulting in a large amount of literature (see \cite{li2014online,dochow2016online} for general expositions). Here we focus on an \emph{unconstrained} setting, allowing both short selling (i.e., holding negative amount of assets) and margin trading (i.e., borrowing money to buy assets). This is related, but different from classical settings in the literature, as discussed in Appendix~\ref{subsection:financial_comparison}.

\subsection{Problem setting}

We consider a market with $d$ assets and discrete trading period $t\in\N_+$. In the $t$-th round, an algorithm chooses a portfolio vector $x_t=[x_{t,1},\ldots,x_{t,d}]\in\R^d$, where $x_{t,i}$ is the \emph{number of shares} of the $i$-th asset that the algorithm suggests to hold. Compared to the previous round, we need to buy $x_{t,i}-x_{t-1,i}$ shares\footnote{W.l.o.g., assume $x_0=x_1$.} (or sell, if negative), which requires paying a $\lambda\abs{x_{t,i}-x_{t-1,i}}$ transaction cost. Then, the market reveals a number $g_{t,i}\in[-G,G]$, which represents the price change per share (of the $i$-th asset) in this round. This effectively increases the value of our portfolio by $\inner{g_t}{x_t}$.

The considered performance metric is the increased amount of \emph{wealth} on any time horizon $[1:T]\subset\N_+$, and such wealth includes the total value of our portfolio \emph{plus cash}. Our goal is to show that the performance of our algorithm is never much worse than that of any unconstrained \emph{Buy-and-Hold} (BAH) strategy, which picks a portfolio vector $u\in\R^d$ at the beginning and holds that amount throughout the considered time horizon. That is, we aim to upper bound $\sum_{t=1}^T\inner{-g_t}{x_t-u}+\lambda\sum_{t=1}^{T-1}\norm{x_t-x_{t+1}}_1$ for all $u\in\R^d$ and $T\in\N_+$. This is exactly the setting of our high dimensional OLO problem (Appendix~\ref{subsection:higherd}) with flipped gradients, therefore if we use our high dimensional OLO algorithm (Algorithm~\ref{alg:higherd}), then the same theoretical result (Theorem~\ref{thm:higherd}) carries over. 

\subsection{Comparison to the rebalancing setting}\label{subsection:financial_comparison}

The online portfolio selection problem has been studied both empirically and theoretically. Most theoretical works with adversarial guarantees consider the \emph{rebalancing} setting, pioneered by Cover \cite{cover1991universal} and followed by a series of works \cite{cover1996universal,helmbold1998line,kalai2002efficient,orseau2017soft,luo2018efficient,mhammedi2022damped,zimmert2022pushing}. Differences to our setting are discussed as follows. 
\begin{enumerate}
\item First, the rebalancing setting forbids short selling (i.e., $x_{t,i}<0$) and margin trading (i.e., borrowing cash to buy an asset), therefore the decision is modeled as a distribution $p_t\in\Delta(d)$ -- the algorithm redistributes its wealth according to this distribution in each round. In contrast, our setting allows both\footnote{Although we only consider the ideal case with zero interest rate on loans.}, so we call it ``unconstrained''. Similar to the loss-regret trade-off in OLO, allowing margin trading introduces a \emph{risk-return trade-off} in some sense: based on its own risk tolerance, one can trade off the best-case return with the worst-case loss on a Pareto-optimal frontier. 
\item Related to the above, existing works consider \emph{Constant Rebalanced Portfolios} (CRP, i.e., $p_t=p^*\in\Delta(d)$) as the benchmark class, and the goal is to lower bound the \emph{ratio} of the growth rate of the considered algorithm to the growth rate of the benchmark. Here we consider unconstrained \emph{Buy-and-Hold} (BAH) strategies as benchmarks, and we aim at an additive bound on the wealth. There have been discussions on the correct choice of benchmarks, but as suggested by a series of works \cite{cover1991universal,helmbold1998line,blum1999universal}, a major weakness of CRPs is the incorporation of transaction costs: such benchmark strategies lose money due to constant rebalancing in every round, which makes the performance guarantee vacuous in certain cases. In contrast, BAH benchmarks do not suffer from this issue. 

\item Finally, transaction costs can take many forms. Here we consider the special case that charges a fixed price \emph{per share}. This is different from the \emph{proportional transaction cost} in some prior works \cite{blum1999universal,gofer2014higher}, which is proportional to the \emph{total value} of the transaction. 
\end{enumerate}

We also note that our Algorithm~\ref{alg:lea} for LEA with switching cost is essentially a comparator adaptive improvement of the \emph{Exponentiated Gradient} (EG) algorithm adopted in \cite{helmbold1998line}. Therefore, it can be applied to the rebalancing setting, following the same argument there. 

\subsection{Synthetic market}\label{subsection:synthetic}

In this subsection, we present numerical results on synthetic markets. 

Two algorithms are tested, our high dimensional OLO algorithm (Algorithm~\ref{alg:higherd}, i.e., ``ours''), and the baseline from \cite{zhang2022adversarial} (its 1D version is surveyed as Algorithm~\ref{algorithm:1d} in Appendix~\ref{subsection:suboptimal}, and we extend it to high dimensions using the same coordinate-wise construction). Both algorithms require a confidence parameter $C$ -- they are both set to $1$ following the practice of comparator adaptive algorithms \cite{orabona2016coin,chen2022better,zhang2022adversarial}. A detailed justification is provided later.

As for the synthetic market, we let $G=1$, $\lambda=0.1$, and the market contains five assets with different return characteristics. Each $g_{t,i}$ is the summation of a i.i.d. noise, a periodic fluctuation and a constant trend. Specifically, we consider three different market return models. The first is
\begin{equation*}
g_{t,1}=0.4\cdot \mathrm{Uniform}[-1,1]+0.4\sin[(t/500)\cdot \pi]+0.2,
\end{equation*}
\begin{equation*}
g_{t,2}=0.5\cdot \mathrm{Uniform}[-1,1]+0.3\sin[(t/500+1/2)\cdot \pi]+0.2,
\end{equation*}
\begin{equation*}
g_{t,3}=0.6\cdot \mathrm{Uniform}[-1,1]+0.2\sin[(t/500+1)\cdot \pi]+0.2,
\end{equation*}
\begin{equation*}
g_{t,4}=0.7\cdot \mathrm{Uniform}[-1,1]+0.1\sin[(t/500+3/2)\cdot \pi]+0.2,
\end{equation*}
\begin{equation*}
g_{t,5}=0.8\cdot \mathrm{Uniform}[-1,1]+0.2.
\end{equation*}
The second model is
\begin{equation*}
g_{t,1}=0.2\cdot \mathrm{Uniform}[-1,1]+0.4\sin[(t/500)\cdot \pi]+0.4,
\end{equation*}
\begin{equation*}
g_{t,2}=0.3\cdot \mathrm{Uniform}[-1,1]+0.3\sin[(t/500+1/2)\cdot \pi]+0.4,
\end{equation*}
\begin{equation*}
g_{t,3}=0.4\cdot \mathrm{Uniform}[-1,1]+0.2\sin[(t/500+1)\cdot \pi]+0.4,
\end{equation*}
\begin{equation*}
g_{t,4}=0.5\cdot \mathrm{Uniform}[-1,1]+0.1\sin[(t/500+3/2)\cdot \pi]+0.4,
\end{equation*}
\begin{equation*}
g_{t,5}=0.55\cdot \mathrm{Uniform}[-1,1]+0.45.
\end{equation*}
The third model is the same as the second one, except we replace $g_{t,5}$ by
\begin{equation*}
g_{t,5}=0.5\cdot \mathrm{Uniform}[-1,1]+0.5.
\end{equation*}

For each market return model, we test both algorithms in 50 random trials, and the increased wealth $\sum_{\tau=1}^t\inner{g_\tau}{x_\tau}-\lambda\sum_{\tau=1}^{t-1}\norm{x_{\tau}-x_{\tau+1}}_1$ (mean $\pm$ std) is plotted in Figure~\ref{fig:synthetic_market}, higher is better. In all three setting, our algorithm beats the baseline by a considerable margin, due to being a lot less conservative. 

\begin{figure}[ht]
     \centering
     \begin{subfigure}[b]{0.32\textwidth}
         \centering
         \includegraphics[width=\textwidth]{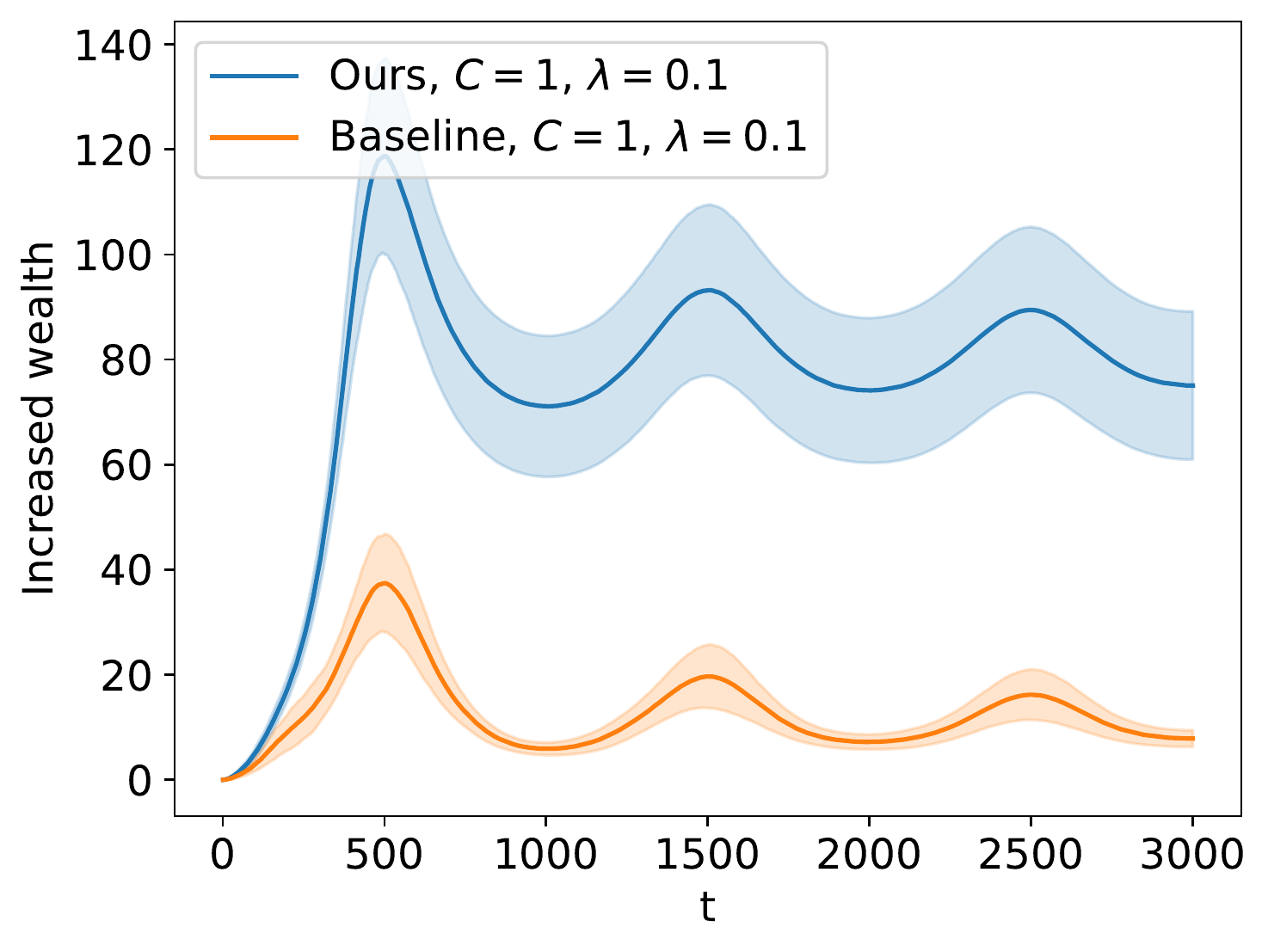}
     \end{subfigure}
     \hfill
     \begin{subfigure}[b]{0.32\textwidth}
         \centering
         \includegraphics[width=\textwidth]{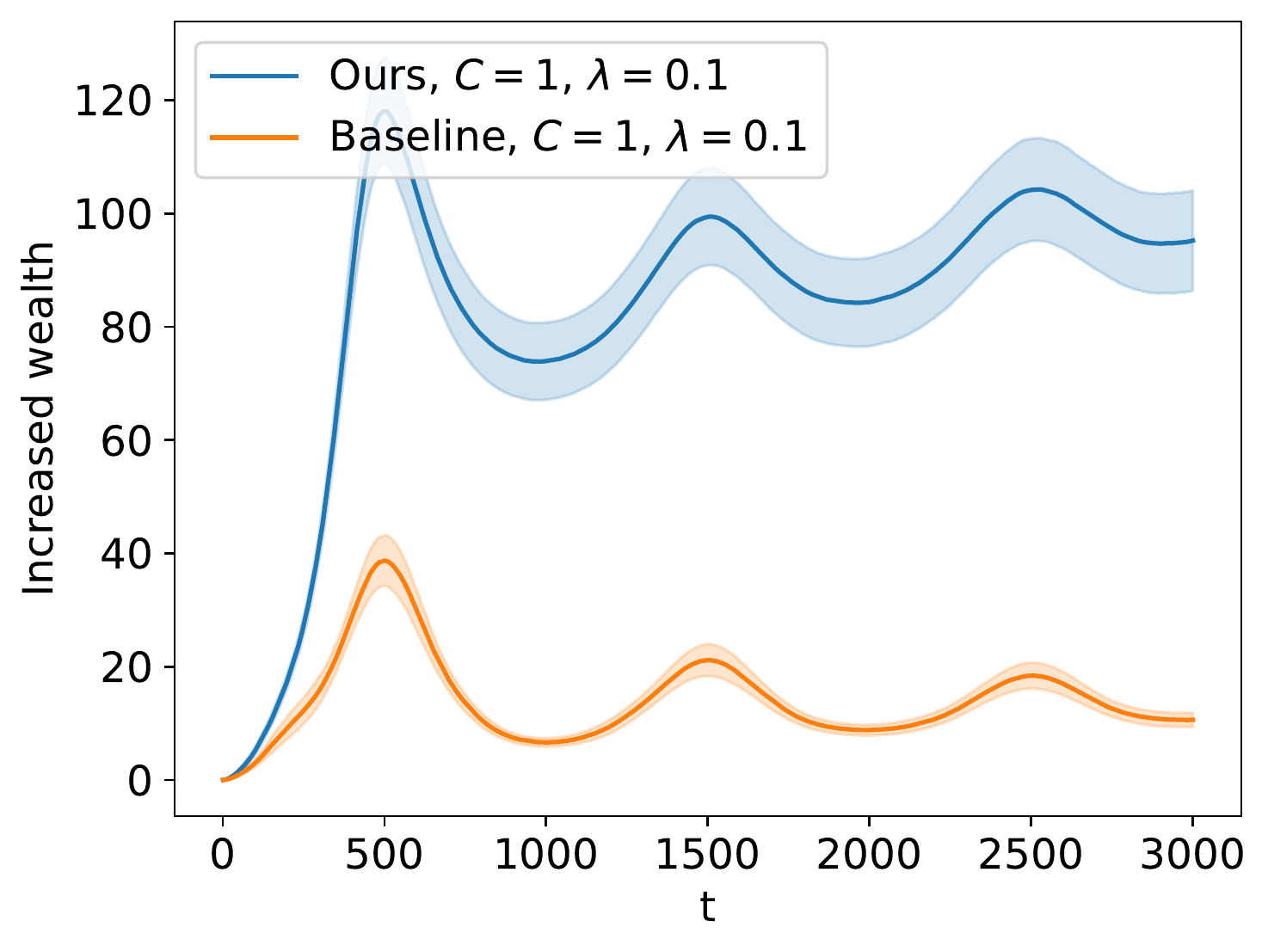}
     \end{subfigure}
     \hfill
     \begin{subfigure}[b]{0.32\textwidth}
         \centering
         \includegraphics[width=\textwidth]{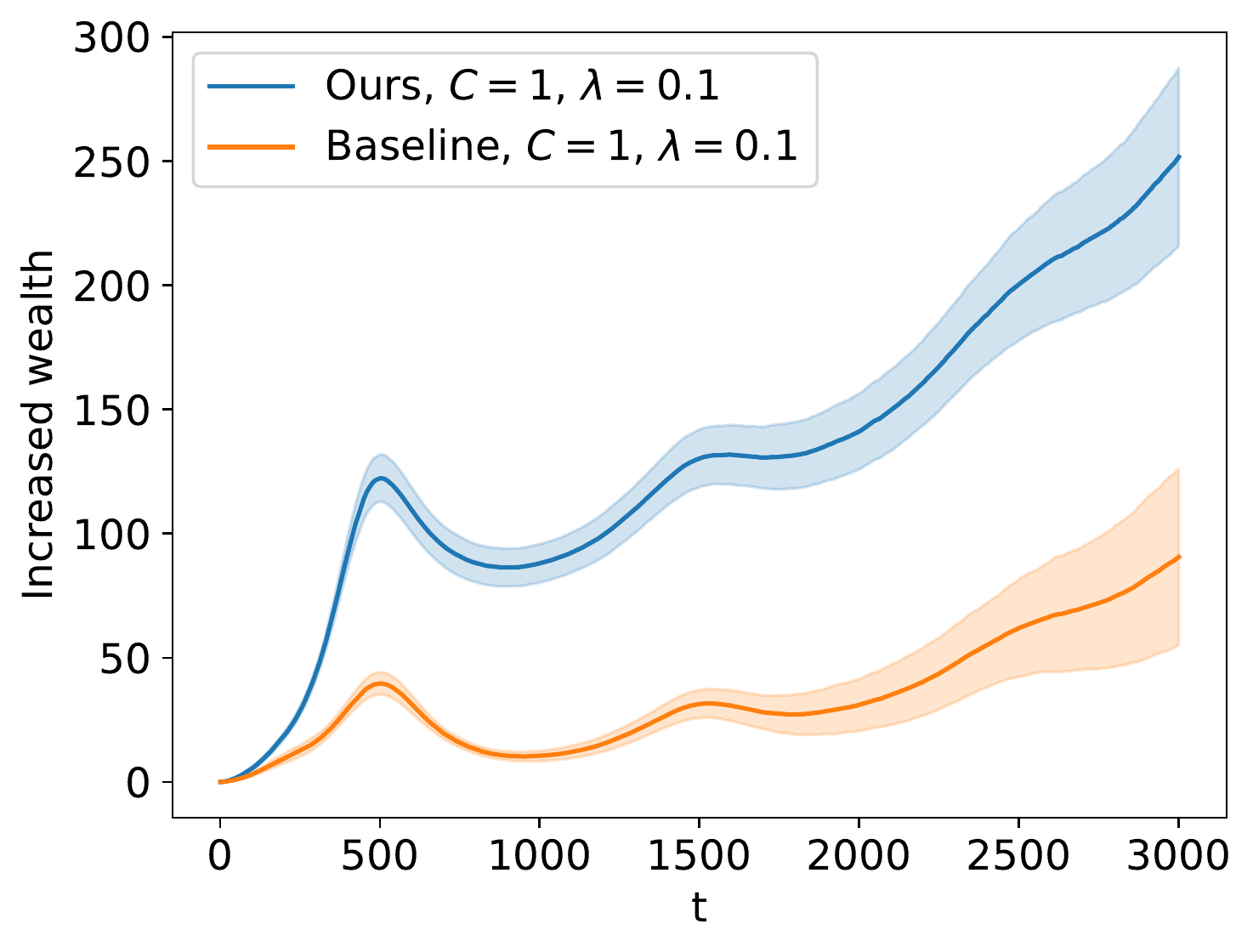}
     \end{subfigure}
    \caption{Synthetic market experiment with different market models. From left to right: the first, the second and the third market model.}
        \label{fig:synthetic_market}
\end{figure}

\paragraph{Discussion on $C$} We remark that setting $C=1$ for both algorithms may create confusion. Let us give it a detailed justification. 

As surveyed in the Introduction, comparator adaptive algorithms are called ``parameter-free'' due to historical reasons. As the name suggests, one may question the existence of \emph{any} hyperparameter in such ``parameter-free'' algorithms. The classical rationale is the following: comparator adaptive regret bounds depend on the hyperparameter $C$ logarithmically, whereas minimax regret bounds depend on the learning rate (and its inverse) linearly. In this regard, comparator adaptive algorithms are provably less sensitive to the correct setup, therefore as a rule of thumb, most practices \cite{orabona2016coin,chen2022better,zhang2022adversarial} simply use $C=1$ without requiring any domain knowledge. Such a default setup removes hyperparameter tuning, which is the most attractive feature of these algorithms in practice. Figure~\ref{fig:synthetic_market} shows the advantage of our algorithm when both algorithms are in this default, parameter-free implementation. 

Nonetheless, for specific tasks like portfolio selection, tuning $C$ can affect the actual performance one cares about (although violating the main purpose of comparator adaptivity). Intuitively, fixing the market, an aggressive trader with a worse strategy could make more profit than a conservative trader with a better strategy. Reflected in our experiment, since the market model does not depend on the invested amount, the baseline with a 5 times larger $C$ simply obtains a 5 times larger increased wealth and beats our algorithm (at certain $t$), cf. Figure~\ref{fig:synthetic_more} (Left). Of course, one can also tune our algorithm with a correspondingly scaled $C$ and beat the baseline again, cf. Figure~\ref{fig:synthetic_more} (Right), just like when both algorithms are in their parameter-free implementation. 

\begin{figure}[ht]
     \centering
     \begin{subfigure}[b]{0.49\textwidth}
         \centering
         \includegraphics[width=0.7\textwidth]{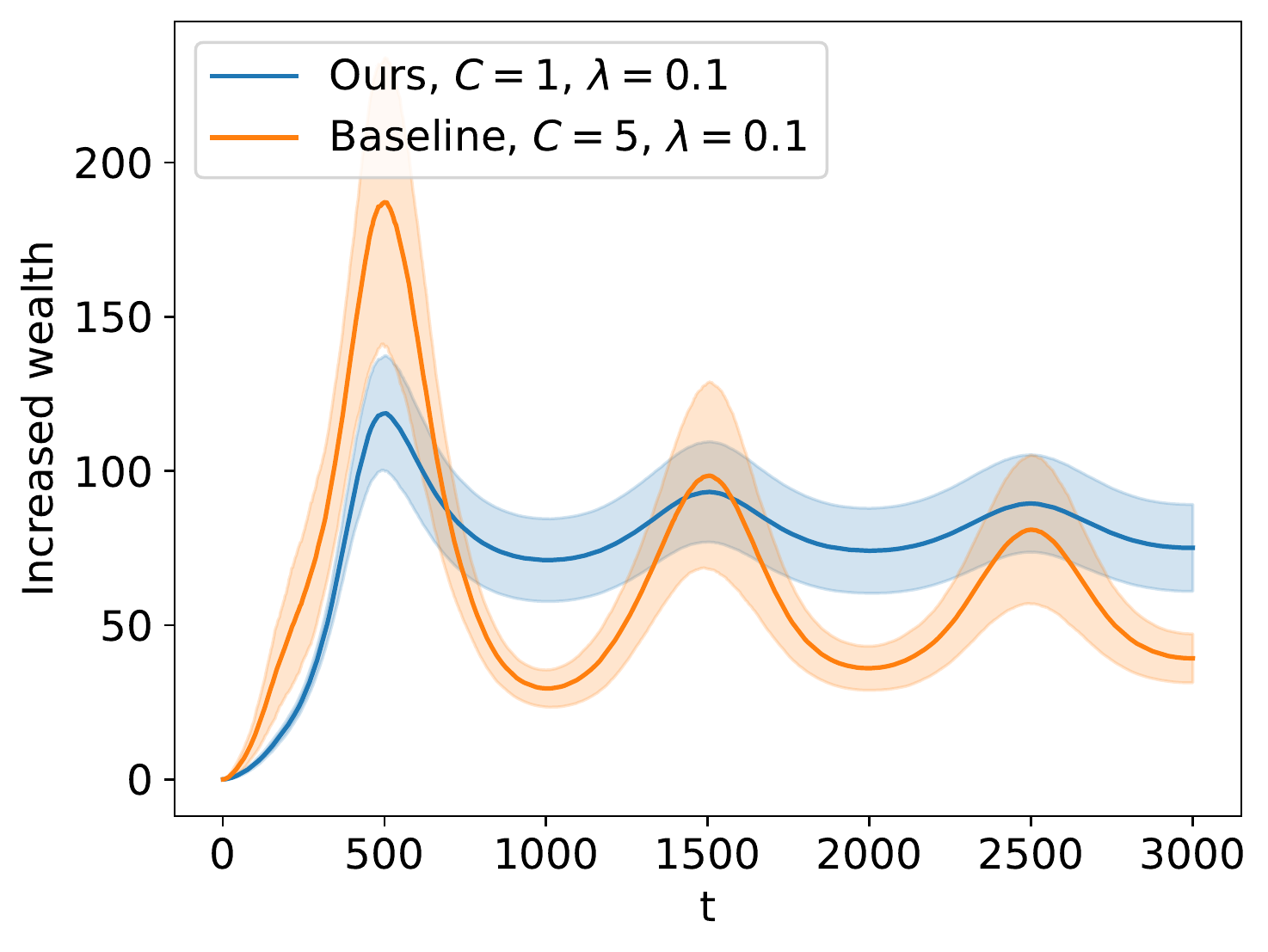}
     \end{subfigure}
     \hfill
     \begin{subfigure}[b]{0.49\textwidth}
         \centering
         \includegraphics[width=0.7\textwidth]{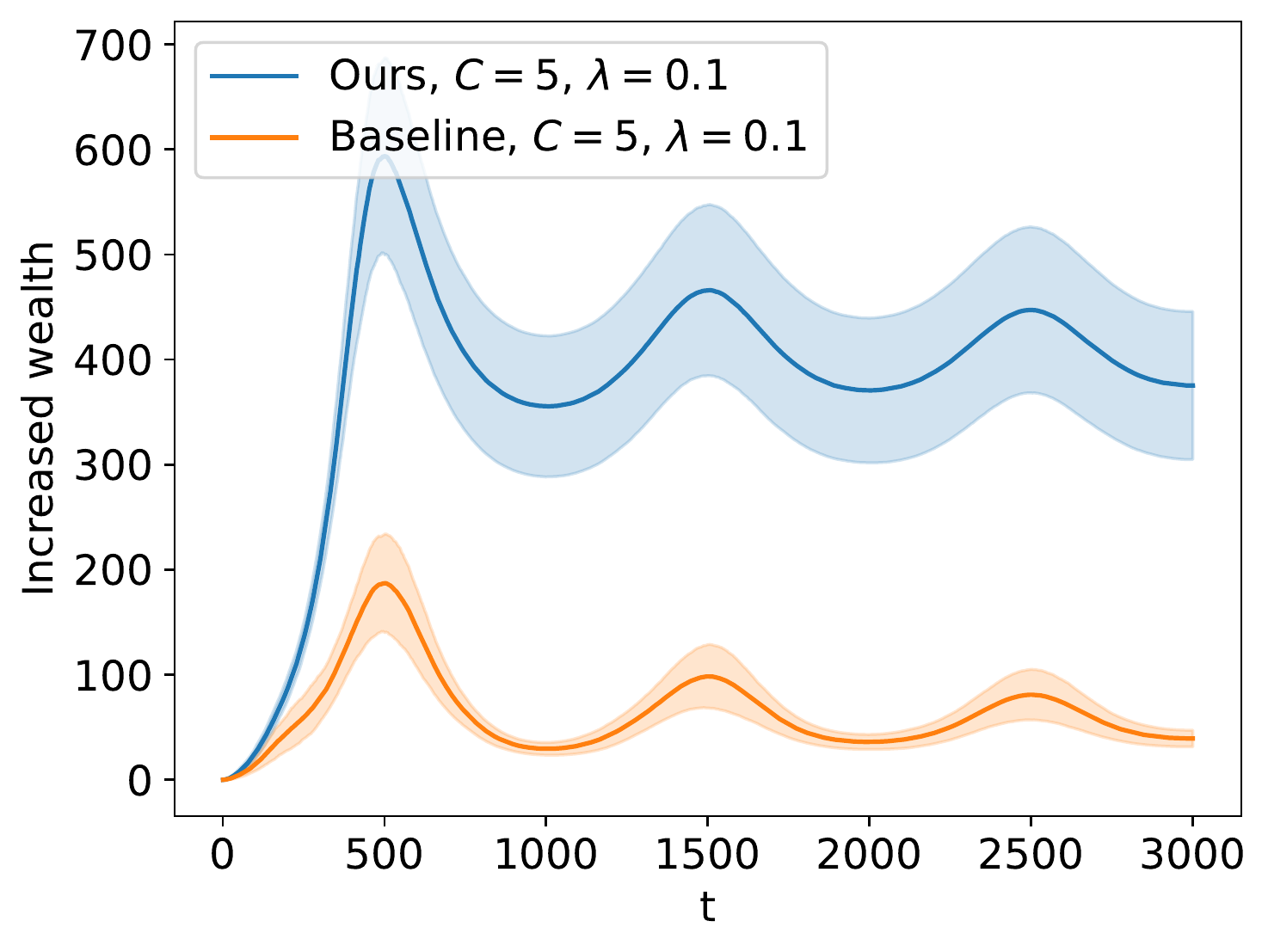}
     \end{subfigure}
    \caption{Synthetic market experiment with tuned $C$; not a parameter-free implementation. Left: only tuning the baseline. Right: tuning both our algorithm and the baseline.}
        \label{fig:synthetic_more}
\end{figure}

Therefore, if tuning $C$ is allowed, then comparing our algorithm to the baseline amounts to comparing two \emph{algorithm classes} both parameterized by $C$. A skeptical reader may wonder if the superior performance of our algorithm in the parameter-free setting is due to the confidence encoding rather than a better algorithm design. That is, is it possible that a baseline with a larger $C$ can consistently outperform our algorithm with $C=1$? We provide evidence against this hypothesis, by increasing $\lambda$ while keeping $C=1$ for our algorithm and $C=5$ for the baseline; results are plotted in Figure~\ref{fig:synthetic_lambda}. It shows that even when the baseline is given an advantage ($C=5$), our algorithm is still better at handling transaction costs due to an improved design. This is aligned with the superiority of our theoretical results. 

\begin{figure}[ht]
     \centering
     \begin{subfigure}[b]{0.32\textwidth}
         \centering
         \includegraphics[width=\textwidth]{Synthetic_tuneC_1.pdf}
     \end{subfigure}
     \hfill
    \begin{subfigure}[b]{0.32\textwidth}
         \centering
         \includegraphics[width=\textwidth]{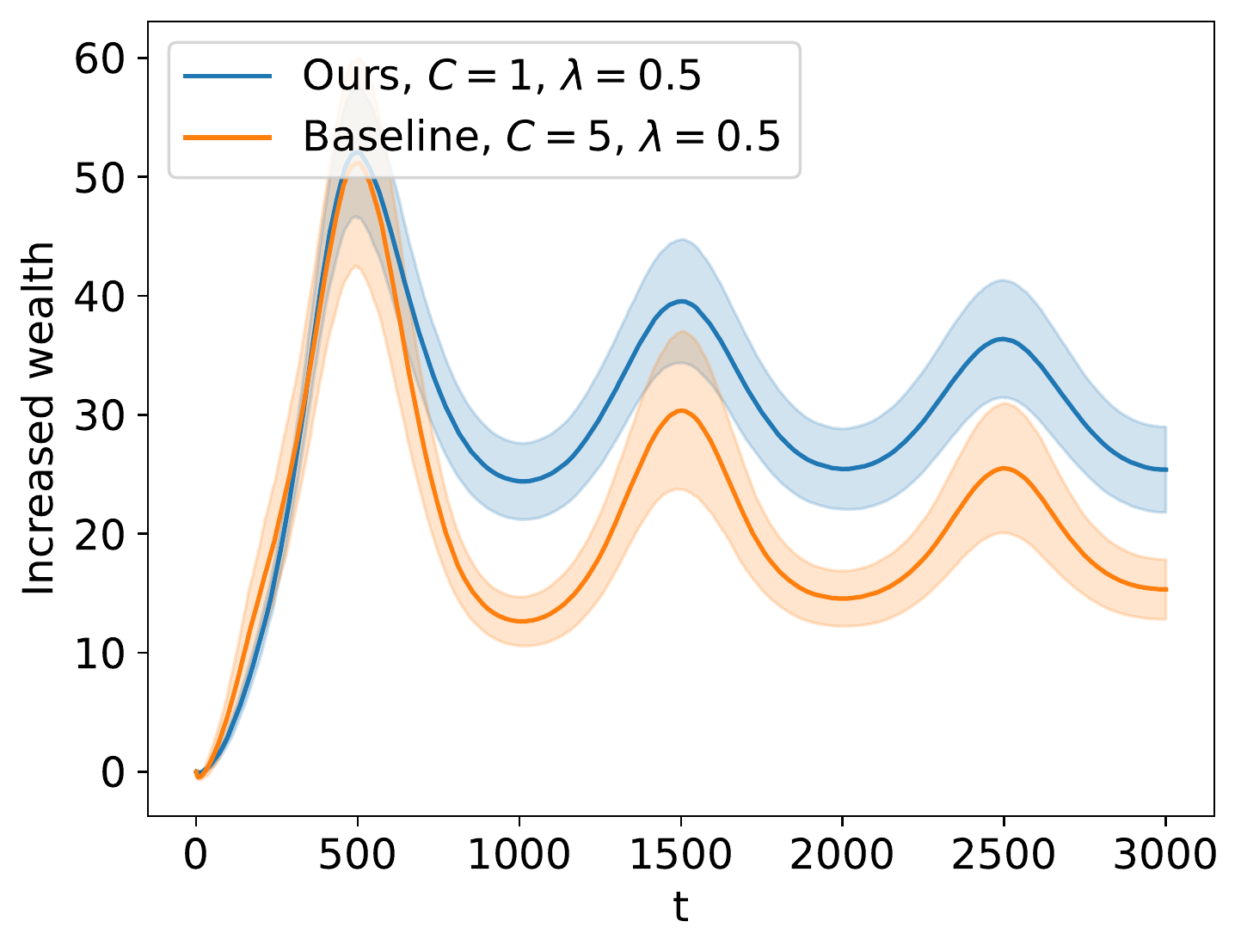}
     \end{subfigure}
     \hfill
     \begin{subfigure}[b]{0.32\textwidth}
         \centering
         \includegraphics[width=\textwidth]{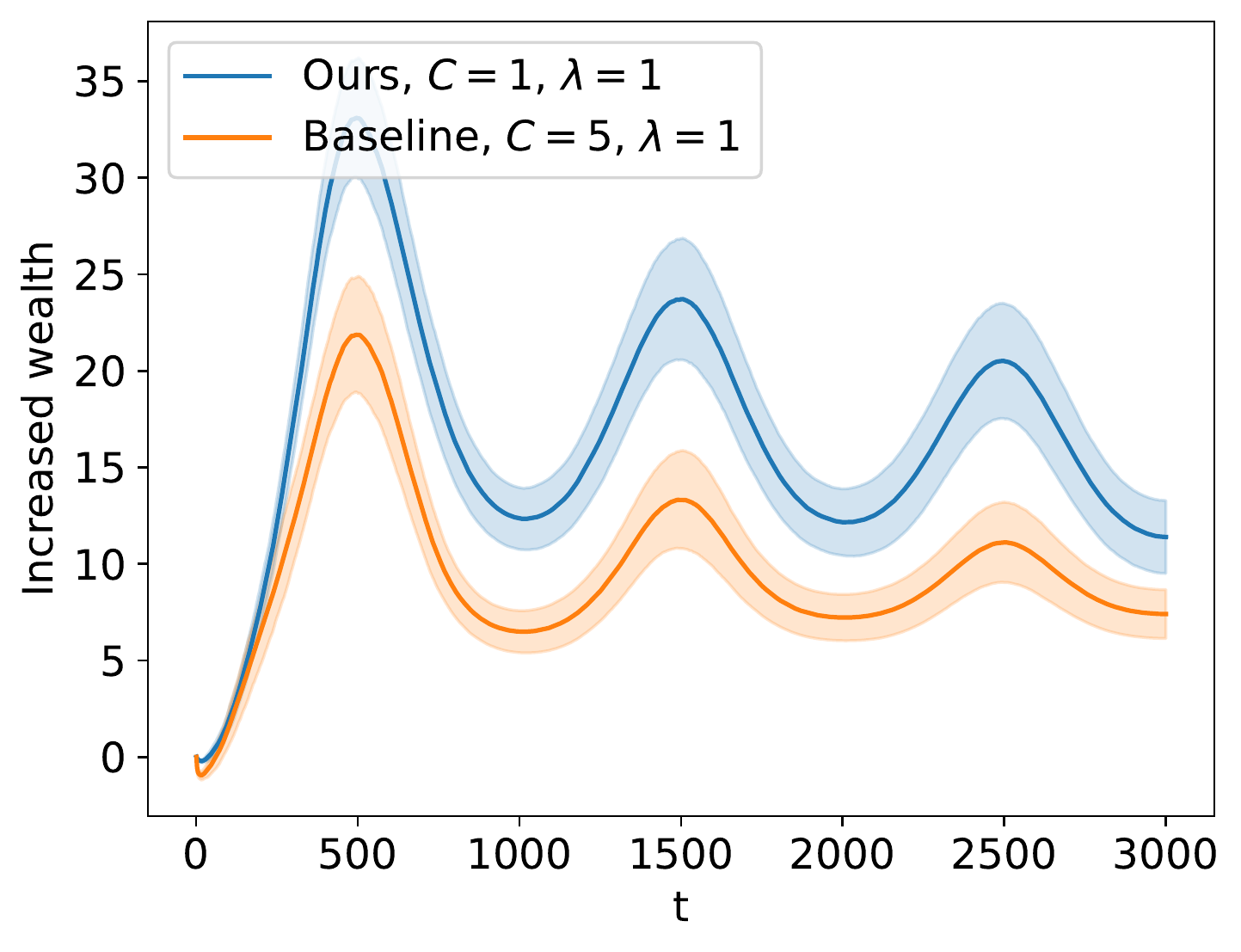}
     \end{subfigure}
        \caption{Synthetic market experiment with increasing $\lambda$. Left: $\lambda=0.1$. Middle: $\lambda=0.5$. Right: $\lambda=1$. The baseline is given an advantage ($C=5$), while our algorithm is in its default parameter-free implementation ($C=1$). It shows our algorithm indeed handles transaction costs better.}
        \label{fig:synthetic_lambda}
\end{figure}

\subsection{Historical stock data}\label{subsection:historical}

Finally, we present some preliminary numerical results on historical US stock data\footnote{US stock price data is publicly available. We retrieved the data from Yahoo Finance website. \url{https://finance.yahoo.com/}}. Eight stocks (Table~\ref{table:1}) are considered on a time period of 5 years (1/1/2013 to 1/1/2018). Our algorithm trades once per day after the market closes, based on the closing price. We take the difference between the closing price on the $(t+1)$-th day and the closing price on the $t$-th day, and define it as the market vector $g_t$. The largest single day price change for any stock is less than \$15, therefore $G$ is set in a posterior manner to 15. We consider a hypothetical broker that charges \$0.1 per share, therefore define $\lambda=0.1$.

\begin{table}[ht]
  \caption{List of considered stocks}
  \label{table:1}
  \centering
  \begin{tabular}{ll}
    \toprule
 Company & Symbol\\    \midrule
 Apple Inc. & AAPL\\
 Berkshire Hathaway Inc. Class B & BRK.B\\
 Meta Platforms Inc. & FB\\
 Johnson \& Johnson & JNJ\\
 JPMorgan Chase \& Co. & JPM\\
 Microsoft Corporation & MSFT\\
 Pfizer Inc. & PFE\\
 Exxon Mobil Corporation & XOM\\
    \bottomrule
  \end{tabular}
\end{table}

Same as the synthetic market experiment, we test our algorithm against the baseline from \cite{zhang2022adversarial}. Our algorithm is in its default parameter-free implementation ($C=1$). However, setting $C=1$ also for the baseline is too conservative, which means the baseline hardly makes any investment, making the comparison less interesting. Therefore we set $C=10$ for the baseline, thus giving it an advantage at the beginning. In this way, the increased wealth of the two algorithms is roughly comparable. 

We plot the results in Figure~\ref{fig:stock}. Specifically, Figure~\ref{fig:stock} (Left) shows the increased wealth (in USD) over the considered time period. Figure~\ref{fig:stock} (Right) shows the cumulative amount of investment (in USD), which is the total net amount of cash the investor uses to buy stocks (i.e., increases when buying, and decreases when selling), plus the transaction costs paid to the broker. Before analyzing this result, we note that such a ``cumulative investment'' only makes sense in our setting, due to a fundamentally different mechanism compared to the rebalancing approach \cite{cover1991universal}: in the latter, the investor is \emph{self-financed}, i.e., it is given a certain budget at the beginning and never adds more money from external sources after that. In contrast, the investor in our setting can add more money at any time it wishes. 

\begin{figure}[ht]
     \centering
     \begin{subfigure}[b]{0.49\textwidth}
         \centering
         \includegraphics[width=0.82\textwidth]{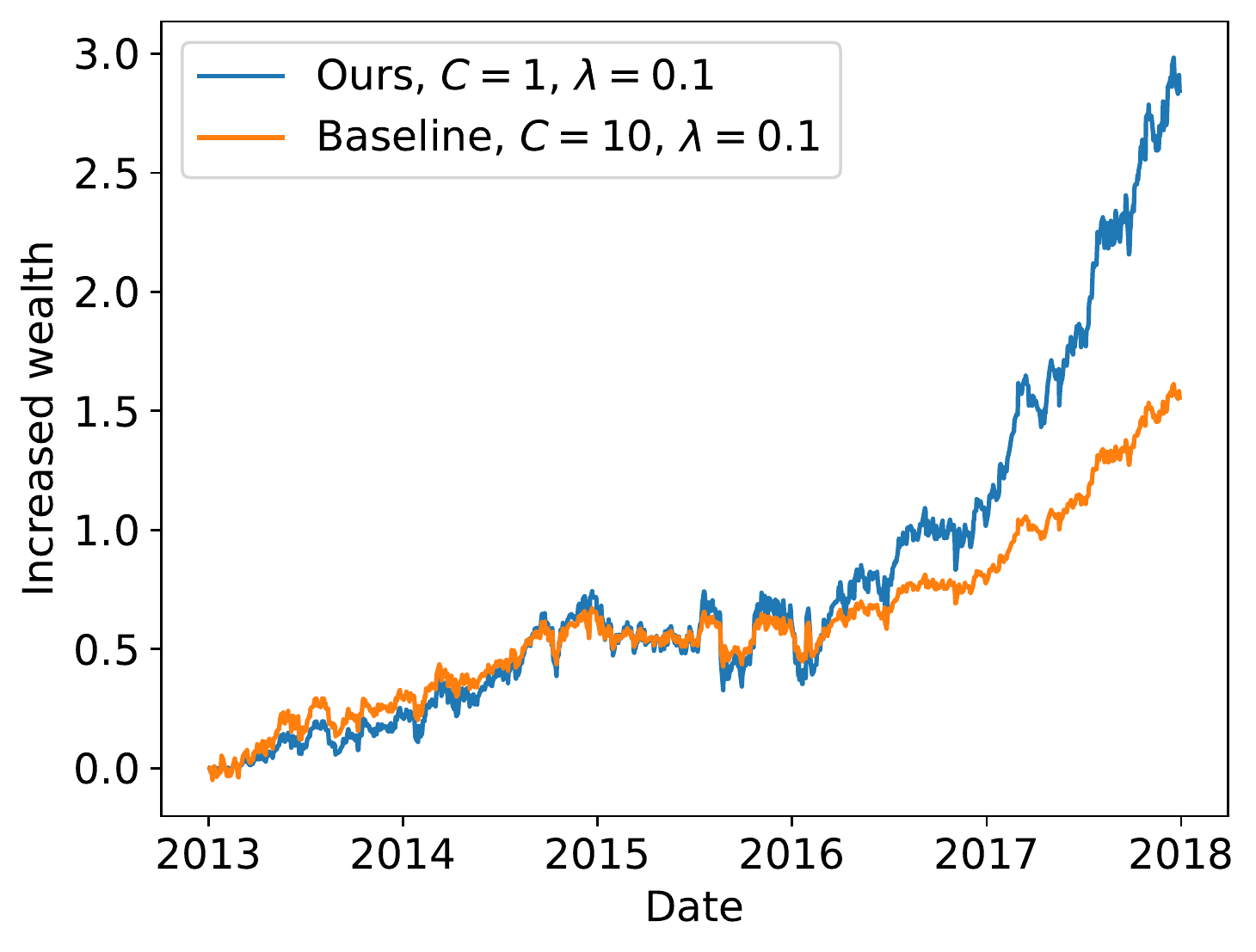}
     \end{subfigure}
     \hfill
     \begin{subfigure}[b]{0.49\textwidth}
         \centering
         \includegraphics[width=0.8\textwidth]{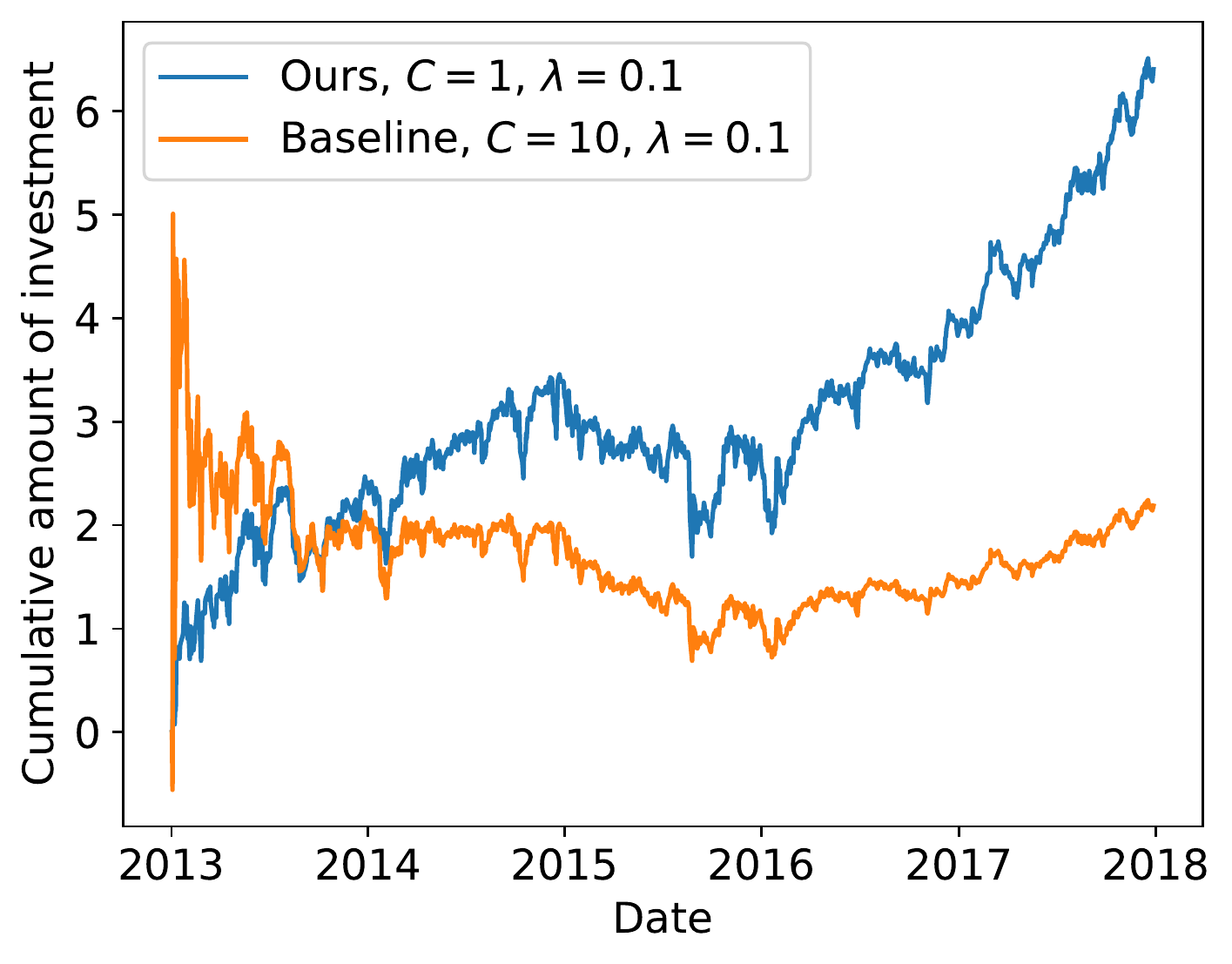}
     \end{subfigure}
    \caption{Experiment on historical US stock data. Left: the increased wealth of the two algorithms. Right: total amount of investment since the start of the experiment (1/1/2013), including the transaction costs paid to the broker. }
        \label{fig:stock}
\end{figure}

From the plot we can see that the baseline is more aggressive at the beginning, due to a much larger $C$. Therefore, it slightly makes more profit during 2013-2014. When the market oscillates and declines in 2015 and 2016, the two algorithms perform roughly the same, while the baseline has a lower risk due to holding a smaller portfolio at the time. However, the major difference starts after mid-2016, when the market grows rapidly. Our algorithm is able to identify this trend and quickly increase the amount of investment. This brings a lot more profit than the baseline, which hardly recovers its confidence from the declining market in the previous year. Such an advantage of our algorithm is partly due to the better control of switching costs, and partly due to a better risk-return trade-off discussed in Appendix~\ref{subsection:conversion}. 

Our experiment also shows a limitation of our unconstrained investment setting. Throughout this five year period, our algorithm invests a total amount of $\sim$\$6.5 (including the transaction costs), and makes a total profit of $\sim$\$3. However, in practice, one typically invests a lot more than this (let's say, \$10,000), and expect a similar rate of return. Our setting does not model such a budget explicitly; instead, it relies on the comparator adaptivity of the trading algorithms to increase the invested amount. Such a process can be slow, especially since we only consider trading once per day. Therefore, to use our algorithm in real trading situations, one has to tune the confidence parameter $C$ to implicitly take his budget and tolerable risk into account. For example, using our algorithm with $C=1000$ would result in investing \$6,500 throughout the five year period, and make a total profit of \$3,000. The connection of this approach to rebalancing could be an interesting direction for future works. 

\end{document}